\documentclass{article}

\usepackage{microtype}
\usepackage{graphicx}
\usepackage{subfigure}
\usepackage{booktabs} % for professional tables
\usepackage{natbib}

\usepackage{hyperref}

\usepackage[accepted]{icml2023}

\usepackage{amsmath}
\usepackage{amssymb}
\usepackage{mathtools}
\usepackage{amsthm}
\usepackage{bbm}
\usepackage{multirow}
\usepackage{pifont}
\newcommand{\cmark}{\ding{51}}%
\newcommand{\xmark}{\ding{55}}%
\usepackage{adjustbox}
\usepackage{resizegather}

\usepackage[capitalize,noabbrev]{cleveref}

\theoremstyle{plain}
\newtheorem{definition}{Definition}
\newtheorem{theorem}{Theorem}
\newtheorem*{theorem*}{Theorem}
\newtheorem{lemma}{Lemma}

\newtheorem{example}{Example}

\newtheorem{proposition}{Proposition}
\theoremstyle{remark}
\newtheorem{remark}{Remark}

\usepackage[textsize=tiny]{todonotes}

\DeclareMathOperator*{\argmin}{argmin}

\newcommand{\notshow}[1]{{}}
\newcommand{\AutoAdjust}[3]{{ \mathchoice{ \left #1 #2  \right #3}{#1 #2 #3}{#1 #2 #3}{#1 #2 #3} }}
\newcommand{\Xcomment}[1]{{}}

\newcommand{\InParentheses}[1]{\AutoAdjust{(}{#1}{)}}
\newcommand{\InBrackets}[1]{\AutoAdjust{[}{#1}{]}}
\newcommand{\InAngles}[1]{\AutoAdjust{\langle}{#1}{\rangle}}
\newcommand{\InNorms}[1]{\AutoAdjust{\|}{#1}{\|}}
\renewcommand{\part}[2]{\frac{\partial #1}{\partial #2}}

\newcommand{\R}{\mathbbm{R}}

\newcommand{\X}{\mathcal{X}}
\newcommand{\Y}{\mathcal{Y}}

\newcommand{\half}{\frac{1}{2}}

\newcommand{\boldx}{\boldsymbol{x}}

\newcommand{\gap}{\textsc{Gap}}
\newcommand{\tgap}{\textsc{TGap}}

\newenvironment{prevproof}[2]{\noindent {\it {Proof of {#1}~\ref{#2}:}}}{$\hfill \square$}

\newcommand{\yangnote}[1]{{\color{cyan}{#1}}}
\icmltitlerunning{Doubly Optimal No-Regret Learning Algorithm for Monotone Games\hfill\thepage}

\begin{document}

\twocolumn[
\icmltitlerunning{Doubly Optimal No-Regret Learning in Monotone Games\hfill\thepage}
\icmltitle{Doubly Optimal No-Regret Learning in Monotone Games}




\begin{icmlauthorlist}
\icmlauthor{Yang Cai}{yyy}
\icmlauthor{Weiqiang Zheng}{yyy}
\end{icmlauthorlist}

\icmlaffiliation{yyy}{Department of Computer Science, Yale University, New Haven, USA}

\icmlcorrespondingauthor{Yang Cai}{yang.cai@yale.edu}
\icmlcorrespondingauthor{Weiqiang Zheng}{weiqiang.zheng@yale.edu}

\icmlkeywords{Machine Learning, ICML}

\vskip 0.3in
]



\printAffiliationsAndNotice{}  % leave blank if no need to mention equal contribution

\begin{abstract}
We consider online learning in  multi-player smooth monotone games.  Existing algorithms have limitations such as (1) being only applicable to strongly monotone games; (2) lacking the no-regret guarantee; (3) having only asymptotic or slow $\mathcal{O}(\frac{1}{\sqrt{T}})$ last-iterate convergence rate to a Nash equilibrium. While the $\mathcal{O}(\frac{1}{\sqrt{T}})$ rate is  tight for a large class of algorithms including the well-studied extragradient algorithm and optimistic gradient algorithm, it is not optimal for all gradient-based algorithms. We propose the \emph{accelerated optimistic gradient} (AOG) algorithm, the first doubly optimal no-regret learning algorithm for smooth monotone games. Namely, our algorithm achieves both (i) the optimal $\mathcal{O}(\sqrt{T})$ regret in the adversarial setting under smooth and convex loss functions and (ii) the optimal $\mathcal{O}(\frac{1}{T})$ last-iterate convergence rate to a Nash equilibrium in multi-player smooth monotone games. As a byproduct of the accelerated last-iterate convergence rate, we further show that each player suffers only an $\mathcal{O}(\log T)$ individual \emph{worst-case dynamic regret}, providing an exponential improvement over the previous state-of-the-art $\mathcal{O}(\sqrt{T})$ bound.
\end{abstract}

\section{Introduction}
We consider multi-agent online learning in games~\citep{cesa-bianchi_prediction_2006}, where the agents engaged in \emph{repeated play} of the same game. In this model, the game (i.e., the agents' payoff functions) is unknown to the agents, and they must learn to play the game through repeated interaction with the other agents. We focus on a rich family of multi-player games -- \emph{monotone games} that has been the central object of a series of recent studies in online learning and optimization~\citep{hsieh2019convergence,golowich2020tight,golowich2020tight,lin_finite-time_2020,hsieh2021adaptive,lin_doubly_2022,cai2022finite}.  Monotone games, first introduced by~\citet{rosen_existence_1965}, encompass many commonly studied games as special cases such as two-player zero-sum games, convex-concave games, $\lambda$-cocoercive games~\citep{lin_finite-time_2020}, zero-sum polymatrix games~\citep{bregman_methods_1987,daskalakis_network_2009,cai2011minmax,cai_zero-sum_2016}, and zero-sum socially-concave games~\citep{even-dar_convergence_2009}. In this paper, we investigate the following fundamental question:

\begin{align*}\label{eq:question star}
&\emph{How fast can  the players' \textbf{day-to-day behavior} converge}\\
&\quad\quad\quad\quad\emph{to a \textbf{Nash equilibrium} in monotone games}\\
& \emph{if players act according to a \textbf{no-regret} learning algorithm?} \tag{*}
\end{align*}

In this context, ``day-to-day behavior'' is used to describe the collective strategy adopted by the agents during each iteration of the repeated game. Our question lies at the heart of the area of learning in games, as illustrated in \citep{fudenberg1998theory, sorin2002first, cesa-bianchi_prediction_2006}. The aim here is to discern whether simple and intuitive learning rules or algorithms (also termed `dynamics' in game theory literature) can lead to the convergence of a joint strategy towards an equilibrium when used by agents to adapt their strategies. The presence of such a learning algorithm offers a logical explanation for the emergence of equilibrium from repetitive interactions, which may not always exhibit complete rationality.

\emph{Regret} is the central metric used in online learning to measure the performance of a learning algorithm. In the classical single-agent setting, online learning considers the following repeated interaction between a player and the environment: (i) at day $t$, the player chooses an action $x_t\in \Omega \subseteq \mathbb{R}^d$; (ii) the environment selects a loss function $f_t(\cdot)$, and the player receives the loss $f_t(x_t)$ along with some feedback (such as the loss function $f_t(\cdot)$, the gradient $\nabla f_t(x_t)$, or just the loss $f_t(x_t)$) and the process repeats. The regret is defined as the difference between the cumulative loss of the player $\sum_{t=1}^T f_t(x_t)$ and the cumulative loss of the best fixed action in hindsight $\min_{x\in \Omega} \sum_{t=1}^T f_t(x)$. A single-agent online learning algorithm is considered \emph{no-regret} if, even under an adversarially chosen sequence of loss functions, its regret at the end of round $T$ is sub-linear in $T$.
 
Arguably, a most common scenario, where the above online learning model instantiates, is multi-agent online learning in games. Namely, every player makes an online decision on their action and receives a loss that is determined based on their own action, as well as the actions chosen by the others. Online learning in repeated games is closely related to various applications in  machine learning. To illustrate, the process of training Generative Adversarial Networks (GANs) can be perceived as a zero-sum game played recurrently between two agents~\cite{arjovsky_wasserstein_2017}. Recent breakthroughs in game-solving, such as AlphaZero~\citep{silver2017mastering}, AI for Stratego~\citep{perolat2022mastering}, leverage self-play, where two agents employ the same learning algorithm to continuously compete against each other, aiming to arrive at a Nash equilibrium. %Online learning in repeated games are related to several machine learning applications. For example, the training of GANs can be viewed as two agents repeatedly playing a zero-sum game~\cite{arjovsky_wasserstein_2017}. Recent advances in game solving (e.g., AlphaZero~\citep{silver2017mastering}, AI for Stratego~\citep{perolat2022mastering}) utilize self-play, where two agents use the same learning algorithm to repeatedly play against each other in order to find Nash equilibria in two-player zero-sum games. Our motivation for proposing an online learning algorithm to solve games stems from the need to better understand the learning algorithms in such repeated interactions.

Do these learning algorithms converge in the repeated game? A well-known result states that if each player uses a no-regret learning algorithm to adapt their action, the empirical frequency of their joint action converges to a coarse correlated equilibrium (CCE)~\citep{cesa-bianchi_prediction_2006}. However, this general convergence result has two caveats: (i) the guaranteed convergence is only the empirical frequency of the players' actions rather than the actual, day-to-day play; and (ii) the concept of CCE has limitations and may violate even the most basic rationalizability axioms~\citep{viossat_no-regret_2013}.\footnote{For instance, a CCE may put positive weight only on strictly dominated actions.} Driven by these dual shortcomings, a significant body of work, as evidenced by various studies~\citep{zhou2017countering,zhou_mirror_2017,zhou2018learning, daskalakis2019last, mokhtari2020unified, hsieh2019convergence, lei2021last,golowich_last_2020,golowich2020tight,lin_finite-time_2020,lin_doubly_2022,cai2022finite}, aims to identify specific types of games as well as no-regret learning algorithms such that the convergence can be strengthened in two principal ways:(a) attaining convergence to the more compelling solution concept of Nash equilibrium, and (b) assuring convergence in the players' day-to-day behaviors, rather than merely in their empirical frequency of actions. In other words, the goal is to pinpoint specific games and devise no-regret learning algorithms so that the players’ action profile converges to a Nash equilibrium in the \emph{last-iterate}.

Monotone games emerge as the most general class of games where such strengthened convergence result is known.\footnote{For the more general family of variationally stable games, only asymptotic convergence to Nash equilibria is known.} Unlike in the general convergence to CCE that holds for any no-regret learning algorithms, the last-iterate convergence to Nash equilibria is more subtle and demands a careful design of the learning algorithm. For example, as demonstrated by~\citet{mertikopoulos_cycles_2018}, the well-known family of no-regret learning algorithms -- \emph{follow-the-regularized-leader} fails to converge even in two-player zero-sum games (a special case of monotone games), as the action profile of the players may cycle in space perpetually. The key to correct such cycling behavior is to introduce \emph{optimism} in the algorithm. Indeed, the optimistic gradient (\ref{OG}) algorithm by \citet{popov_modification_1980}, a optimistic variant of the gradient descent algorithm, has recently been shown to exhibit an $\mathcal{O}(\frac{1}{\sqrt{T}})$ last-iterate convergence rate to a Nash equilibrium in monotone games~\citep{golowich2020tight,cai2022finite,golowich2020last}. As shown by \citet{golowich2020tight}, this rate is tight for \ref{OG}. However, it is not clear if $\mathcal{O}(\frac{1}{\sqrt{T}})$ is the optimal rate achievable by a no-regret algorithm.

\subsection{Our Contributions} We consider multi-agent online learning in monotone games with \emph{gradient feedback}. More concretely, each player $i$ at day $t$ not only observes their loss $\ell^i(x^i_t,x^{-i}_t)$ but also receives the gradient $\nabla_{x^i_t} \ell^i(x^i_t,x^{-i}_t)$.

\paragraph{Main Contribution}
We answer question \eqref{eq:question star} by presenting a new single-agent online learning algorithm -- the \emph{Accelerated Optimistic Gradient}~(\ref{AOG}) that is \emph{doubly optimal} (\Cref{thm:adaptive-optimal}). More specifically, 
\begin{description}
    \item [Optimal regret:] \ref{AOG} achieves the optimal $\mathcal{O}(\sqrt{T})$-regret in the adversarial environment;
    \item[Optimal last-iterate convergence rate:] If all players use \ref{AOG} to determine their actions in a monotone game, the action profile has the optimal $\mathcal{O}(\frac{1}{T})$ last-iterate convergence rate to a Nash equilibrium.
\end{description}

Note that $\mathcal{O}(\frac{1}{T})$ is the fastest rate possible for solving monotone games using any gradient-based methods~\citep{ouyang2021lower,yoon2021accelerated}.\footnote{These lower bounds apply to general first-order methods that produce their iterates in an arbitrary manner based on past gradient information.} Since the players only receive gradient feedback in our setting, this lower bound also applies to our problem.

\paragraph{Step-size adaptation.} We provide an implementation of \ref{AOG} (\Cref{alg}) that can automatically adapt to the environment and achieves a best-of-both-world guarantee. When deploy in an adversarial setting, \Cref{alg} obtains at most $\mathcal{O}(\sqrt{T})$-regret; when deploy in a monotone game where other players also play according to \Cref{alg}, the action profile converges to a Nash equilibrium at a $\mathcal{O}(\frac{1}{T})$ rate in the last-iterate. Importantly, the adaptation does not require any communication between the players and only uses the the player's local information. We believe such guarantee is crucial as even in a game setting, other players may not follow the same algorithm and might act arbitrarily, in which case, our algorithm still provides a guarantee on the worst-case regret.

\paragraph{Dynamic regret.} As an interesting byproduct of our last-iterate convergence rate, we further show that each player suffers only an $\mathcal{O}(\log T)$ individual \emph{dynamic regret}, when all players play according to \Cref{alg} (\Cref{thm:dynamic-regret}). The dynamic regret of an algorithm is defined as the difference between the algorithm's cumulative loss and the cumulative loss of the best action every day. The dynamic regret is notoriously difficulty to minimize, and it is well-known that a linear dynamic regret is unavoidable in the adversarial setting. In the game setting, results on dynamic regret are also sparse. To the best of our knowledge, the only sub-linear dynamic regret bound we are aware of is the $\mathcal{O}(\sqrt{T})$ dynamic regret of \ref{OG} for monotone games. Our accelerated algorithm obtains an exponential improvement on the dynamic regret. See \Cref{tab:my-table} for comparison with other well-studied learning algorithms in monotone games.

\begin{table}[H]
\centering
\begin{tabular}{cccc}
\hline
\multirow{2}{*}{Algorithm} & Adversarial Setting & \multicolumn{2}{c}{Monotone Games}                      \\ \cline{2-4} 
                           & No-Regret?    & Rate$^{*}$                          & D-Regret$^{**}$         \\ \hline
GD                         & \cmark        & \xmark                            & $\Omega(T)$             \\ \hline
EG                         & \xmark        & $\mathcal{O}(\frac{1}{\sqrt{T}})$ & $\mathcal{O}(\sqrt{T})$ \\ \hline
OG                         & \cmark        & $\mathcal{O}(\frac{1}{\sqrt{T}})$ & $\mathcal{O}(\sqrt{T})$ \\ \hline
EAG                        & \xmark        & $\mathcal{O}(\frac{1}{T})$        & $\mathcal{O}(\log T)$   \\ \hline
{\bf This paper}           & \cmark        & $\mathcal{O}(\frac{1}{T})$        & $\mathcal{O}(\log T)$   \\ \hline
\end{tabular}
\caption{Existing results on learning in monotone games. (*) last-iterate convergence rate with respect to the gap function. (**) individual worst-case dynamic regret in monotone games.}
\label{tab:my-table}
\end{table}

\paragraph{Technique.} The key of our new algorithm is combining \emph{optimism} with \emph{Halpern iteration}~\citep{halpern1967fixed}, a mechanism used in optimization to design accelerated methods. In our setting, Halpern iteration can be viewed as adding a diminishing strongly convex loss to the player's loss function. The schedule used to decrease the added loss must be crafted carefully. If the added loss diminishes too slowly, the adversarial regret would be sub-optimal; if the added loss decreases too quickly, the algorithm may converge at a slower rate. The Halpern iteration provides a schedule that strikes the right balance and allows us to obtain the doubly optimal algorithm.

\section{Preliminaries}
\paragraph{Basic Notation.} We consider Euclidean space $(\R^n, \InNorms{\cdot})$ where  $\InNorms{\cdot}$ is $\ell_2$-norm. We say a set $\X \subseteq \R^n$ is bounded by $D  >0$ if $\InNorms{x-x'} \le D$ for any $x,x'\in \X$. Given a closed and convex set $\X \subseteq \R^n$, the Euclidean projection operator is $\Pi_\X : \R^n \rightarrow \X$ such that $\Pi_\X [x] = \argmin_{x' \in \X} \InNorms{x-x'}$. For closed and convex set $\X$, Euclidean projection is \emph{non-expansive}, i.e., $\InNorms{ \Pi_\X[x] - \Pi_\X[x']}\le \InNorms{x-x'}$.  For a closed convex set $\X$, the normal cone of $x\in \X$ is defined as $N_\X(x):=\{v: \InAngles{v,x'-x} \le 0\}$. We make use of the following properties of the normal cone: (i) for any $v \in N_\X(x)$, $x = \Pi_\X[x+v]$; (ii) if $x = \Pi_\X[x']$, then $x'-x\in N_\X(x)$.

\subsection{Monotone Games and Nash Equilibria}
A (continuous)  multi-player game is denoted as $\mathcal{G} = ([N], (\X^i)_{i\in [N]}, (\ell^i)_{i \in [N]})$ where $[N] = \{1,2,\ldots,N\}$ denotes the set of players. Each player $i$ chooses action from a compact and convex set $\mathcal{X}^i \in \R^{n_i}$ and we write $\X = \prod_{i=1}^N \X^i \in \R^n$ where $n = n_1 + \cdots + n_N$. We always use $x^{-i}$ to denote the actions of all players except player $i$ and write $\boldx = (x^i, x^{-i}) = (x^1, x^2 \ldots, x^N)$ as players'  \emph{action profile} or \emph{strategy profile}. Note that we reserve the bold $\boldx$ to denote the players' action profile and use the normal $x$ to denote a single player's action. Each player $i$ wishes to minimize a loss function $\ell^i(x^i, x^{-i}) : \X \rightarrow \R$ which is continuous in $\boldx$ and convex in $x^i$. In this paper, we study learning in multi-player games with gradient feedback where after playing action profile $\boldx$, each player $i$ receives $V^i(\boldx) := \nabla_{x^i} \ell^i(x^i,x^{-i})$. We define the gradient operator $V:\X \rightarrow \R^n$ to be $ V(\cdot)= (V^1(\cdot) \cdots, V^N(\cdot))$. 
The widely used solution concept for a game is \emph{Nash equilibrium}, an action profile where no player gains from unilateral deviation. Formally, 
a Nash equilibrium of a game $\mathcal{G}$ is an action profile $\boldx_\star \in \X$ such that for each player $i$, it holds that $\ell^i(\boldx_\star) \le \ell^i(x^i, x_\star^{-i})$ for any $x^i \in \X^i$.

In this paper, we study \emph{smooth monotone} games where the gradient operator $V$ is \emph{$L$-Lipschitz} for $L > 0$:
\[\InNorms{V(\boldx)-V(\boldx')} \le L \cdot \InNorms{\boldx-\boldx'}, \forall \boldx,\boldx'\in\X,\]
and \emph{monotone}~\citep{rosen_existence_1965} :
\[
\InAngles{V(\boldx)-V(\boldx'), \boldx-\boldx'} \ge 0, \forall \boldx,\boldx' \in \X.
\]
It is not hard to see that for smooth monotone games, a Nash equilibrium always exists. If $\boldx_\star$ is a Nash equilibrium, then a simple characterization of $\boldx_\star$ is that, for any $\boldx \in \X$, it holds that $\InAngles{V(\boldx_\star),  \boldx_\star - \boldx}  \le 0.$

Monotone games include many well-studied games, e.g., two-player zero-sum games, convex-concave games, $\lambda$-cocoercive games~\citep{lin_finite-time_2020}, strongly monotone games (such as Kelly auctions), zero-sum polymatrix games~\citep{bregman_methods_1987,daskalakis_network_2009,cai2011minmax,cai_zero-sum_2016}, and zero-sum socially-concave games~\citep{even-dar_convergence_2009}. 

\begin{example}[Convex-Concave Min-Max Optimization]
\label{ex:min-max}
Given a function $f(x,y): \X\times\Y \rightarrow \R$ that is convex in $x$ and concave in $y$, find a \emph{saddle point} $z=(x,y)$ such that $
f(x, y') \le f(x,y) \le f(x', y), \forall x' \in \X, y' \in \Y$.
It is not hard to see that the set of Nash equilibria of a two-player zero-sum game $\mathcal{G} = \{[2], (\X, \Y), (f,-f)\}$ corresponds to the set of saddle points of $f$. Thus convex-concave min-max optimization is a special case of monotone games.
\end{example}

For a monotone game $\mathcal{G}$ and an action profile $x$, two standard measures of proximity to Nash equilibrium are the \emph{gap} function and the \emph{total gap} function. 
\begin{definition} Let $\mathcal{G} = ([N],(\X^i)_{i \in [N]},(\ell^i)_{i\in [N]})$ be a monotone game. The gap function for $\boldx \in \X$ is $\gap(\boldx) = \max_{\boldx'\in \X} \InAngles{V(\boldx), \boldx -\boldx'}$.
The total gap function for $\boldx \in \X$ is $\tgap(\boldx) = \sum_{i=1}^N \InParentheses{ \ell^i(\boldx) - \min_{x' \in \X^i} \ell^i(x', x^{-i})}$.
Since $\ell^i$ is convex in $x^i$ for all $i \in N$, we have $\tgap(\boldx) \le \gap(\boldx)$ for all $\boldx \in \X$.
\end{definition}
A stronger measure of proximity to Nash equilibrium is the \emph{tangent residual} defined as $r^{tan}(\boldx) = \min_{c \in N_\X(\boldx)} \InNorms{V(\boldx) + c}$. The tangent residual is an upper bound for both the gap and the total gap. 
\begin{lemma}[\cite{cai2022finite}]
\label{lem:gap bounded by residual}
Let $\mathcal{G} = ([N], (\X^i)_{i\in[N]}, (\ell^i)_{i\in [N]})$ be a monotone game where $\X = \prod_{i\in [N]}\X^i$ is bounded by $D$. For any $\boldx \in \X$, we have $\tgap(\boldx) \le  \gap(\boldx) \le D \cdot r^{tan}(\boldx)$.
\end{lemma}

\subsection{Online Learning and Regret}
A central theme of online learning is to design learning algorithms that minimize the \emph{regret}. For each time $t = 1, 2, \ldots, T$,  suppose the environment generates convex loss function $f_t : \Omega \rightarrow \R$ and the algorithm chooses action $x_t \in \Omega$ where $\Omega \subseteq \R^d$ is a compact convex set. The \emph{external regret} is defined as the gap between the algorithm's realized cumulative loss and the cumulative loss of the best fixed action in hindsight: $\mathrm{Reg}(T):= \sum_{t=1}^T f_t(x_t) - \min_{x \in \Omega} \sum_{t=1}^T f_t(x)$.
By convexity of $\ell_t$, we can bound the external regret by $ \mathrm{Reg}(T) \le\max_{x\in \Omega}\sum_{t=1}^T \InAngles{\nabla f_t(x_t), x_t - x}$.
We will simply call the external regret as regret and any algorithm achieving sub-linear regret $\mathrm{Reg}(T) = o(T)$ as a \emph{no-regret} algorithm. 

A much stronger performance measure of an online algorithm is the (worst-case) \emph{dynamic regret} \citep{zinkevich2003online}:  $ \mathrm{DynamicReg}(T):= \sum_{t=1}^T f_t(x_t) - \sum_{t=1}^T \min_{x \in \Omega} f_t(x)$, where the algorithm is competing with the best action in each round. It is not hard to see that in adversarial setting, $\mathrm{DynamicReg}(T)$ must be linear in $T$.

\section{No-Regret Learning Algorithms and Games}
In this section, we first review some background of gradient-based algorithms from both the online learning and optimization. 

We start with \emph{online gradient descent} (GD)~\citep{zinkevich2003online}: the algorithm produces iterates $x_t \in \Omega$ defined by $x_{t+1} = \Pi_{\Omega} \InBrackets{x_t - \eta_t g_{t}}$ where we write $g_t := \nabla f_t(x_t)$ as the gradient of the loss function $f_t$. Online gradient descent is a no-regret algorithm in the adversarial setting. When employed by all players, however, it \emph{diverges} in last-iterate even for simple two-player zero-sum games.

\paragraph{Optimism in Online Learning} A modification of online gradient descent is the \emph{Optimistic Gradient~\eqref{OG}}~\citep{popov_modification_1980,rakhlin2013optimization,daskalakis2018training}: in each round $t$, the algorithm chooses action $x_{t+\half}$, receives $g_{t+\half} := \nabla f_t(x_{t+\half})$, and updates iterates: 
\begin{equation}
\label{OG}
\tag{OG}
    \begin{aligned}
    x_{t+\half} &= \Pi_{\Omega} \InBrackets{x_t - \eta_t g_{t-\half}}, \\
    x_{t+1} &= \Pi_{\Omega} \InBrackets{x_t - \eta_t g_{t+\half}}.
\end{aligned}
\end{equation}
Compared to online gradient descent, \ref{OG} also achieves optimal regret in the single-agent adversarial setting. Moreover, \ref{OG} converges in the last-iterate sense as optimism stabilizes the trajectory. When employed by all players in monotone games, their trajectory of play $(\boldx_{t+\half})_{t \ge 1}$ converges to a Nash equilibrium with an $\mathcal{O}(\frac{1}{\sqrt{T}})$ last-iterate convergence rate~\citep{cai2022finite}. Unfortunately, the $\mathcal{O}(\frac{1}{\sqrt{T}})$ rate is tight for \ref{OG} and more generally all p-SCLI algorithms~\citep{golowich2020tight}. New ideas are needed to further sharpen the convergence rate.

\paragraph{Acceleration in Optimization} We are inspired by a technique from optimization for accelerating first-order methods known as the \emph{Halpern iteration}~\citep{halpern1967fixed} or \emph{Anchoring}. The technique is closely related to Nesterov's accelerated method~\citep{tran2022connection} and has received extensive attention from the optimization community recently~\citep{Diakonikolas2020halpern, yoon2021accelerated, Lee2021fast,cai2022accelerated}. When the Halpern iteration is applied to the classical extragradient (EG) algorithm~\citep{korpelevich_extragradient_1976}, which belongs to the p-SCLI family and also has an $\mathcal{O}(\frac{1}{\sqrt{T}})$ last-iterate convergence rate~\citep{cai2022finite}, the resulting extra anchored gradient (EAG) algorithm achieves an $\mathcal{O}(\frac{1}{T})$ last-iterate convergence rate~\citep{yoon2021accelerated,cai2022accelerated}. \citet{cai2023accelerated} obtain a single-call algorithm -- Accelerated Reflected Gradient (ARG) that also achieves the same optimal last-iterate convergence rate. However, EAG is not suitable for multi-player games, as it could exhibit linear regret as we demonstrated in  Appendix~\ref{app:EAG}. ARG requires evaluating the gradient at points outside of the feasible domain, thus it is also incompatible with multi-player games.  Our analysis is based on a construction from~\citep{golowich2020tight}, where they show that EG has linear regret in multi-player games.

\subsection{Accelerated Optimistic Gradient}
We propose the following algorithm -- the \emph{accelerated optimistic gradient \eqref{AOG}} algorithm. The central idea is to combine \emph{optimism} with \emph{Halpern iteration}: in round $t$, the algorithm chooses action $x_{t+\half}$ and updates as follows.
\begin{equation}
\label{AOG}
\tag{AOG}
    \begin{aligned}
    x_{t+\half} &= \Pi_{\Omega} \InBrackets{ x_t - \eta_t g_{t-\half}  + \frac{1}{t+1}(x_1 -x_t)}, \\
    x_{t+1} &= \Pi_{\Omega} \InBrackets{ x_{t} - \eta_{t} g_{t+\half} + \frac{1}{t+1}(x_1 -x_t)}.
\end{aligned}
\end{equation}

\paragraph{Double Optimality.}
Our main result is that \eqref{AOG} is a doubly optimal online algorithm: with $\eta_t = \Theta(\frac{1}{\sqrt{t}})$, \eqref{AOG} achieves optimal $\mathcal{O}(\sqrt{T})$ regret in adversarial setting (\Cref{thm:regret}); when all players employ \eqref{AOG} with constant step size in a monotone game, their trajectory of play enjoys optimal $\mathcal{O}(\frac{1}{T})$ last-iterate convergence rate (\Cref{thm:last-iterate rate}).

\paragraph{Step-Size Adaptation} 
We also present an implementation of \eqref{AOG} in \Cref{alg} with a step-size adaptation procedure (Line 7-11). This procedure uses the player's own \emph{second-order gradient variation} $S_{t+1} = \sum_{s=2}^t \InNorms{g_{s+\half} - g_{s-\half}}^2$ as a proxy for the environment and adapts the step-size accordingly. The high level idea is that if all players use \Cref{alg} in a smooth monotone game, then each player's second-order gradient variation remains to be bounded by a constant that only depends on $L$ and $D$
(\Cref{thm:second-order gradient variation}), so the algorithm will keep a constant learning rate and achieve an $\mathcal{O}(\frac{1}{T})$ last-iterate convergence (\Cref{thm:last-iterate rate}); if the player's second-order gradient variation exceeds a certain constant threshold, then \Cref{alg} decreases the learning rate according to the second-order gradient variation, and by the standard argument of "regret is bounded by stability", we can essentially bound the player's regret by the the second-order gradient variation, which is at most $\mathcal{O}(\sqrt{T})$ even in the adversarial setting (\Cref{thm:regret}). 

\notshow{We face two challenges in designing a doubly optimal algorithm. In order to achieve optimal $\mathcal{O}(\sqrt{T})$ regret against possibly adversarial opponents, the step size $\eta^i_t$ must decrease at a rate $\mathcal{O}(\frac{1}{\sqrt{t}})$. Using constant step size in the adversarial setting leads to significant unstability of strategies and thus linear regret. However, any step size without a constant lower bound leads to slower than $\mathcal{O}(\frac{1}{T})$ last-iterate convergence rate. The key to double optimality is using second-order path length\footnote{Directly checking the tangent residual or the gap function is computationally more demanding. In contrast, we can use upper bound on the tangent residual to bound the second-order path length, which is easy to compute. } as a proxy: in both single-agent and multi-agent settings, if a player's second-order path length $P^i_{t+1} = \sum_{s=1}^t \InNorms{g^i_{s+\half} - g^i_{s-\half}}^2$ is bounded by a constant, then player $i$'s regret is $\Tilde{O}(1)$; in multi-agent setting, if all players use AOG, then each player's second-order path length is bounded by a constant (depends only on $L$ and $D$). The former argument is standard and known as "regret is bounded by stablity". The latter argument on constant second-order path length is non-trivial and relies on the optimal $\mathcal{O}(\frac{1}{T})$ last-iterate convergence rate. Moreover, the $\mathcal{O}(\frac{1}{T})$ last-iterate convergence rate also implies each player enjoys a $\Tilde{O}(1)$ \emph{worst-case dynamic regret}. The dynamic regret guarantee complements and strengthens existing results on $\Tilde{O}(1)$ \emph{external regret} or \emph{swap regret} for learning in games.}

\begin{algorithm}
\caption{AOG with step-size adaptation}
\begin{algorithmic}[1]\label{alg}
\STATE {\bfseries Input:}  $L, D > 0$. 
\STATE Initialize $g_{\frac{1}{2}} = \Vec{0}$, $\eta_1 = \eta = \frac{1}{3L}$, and choose an arbitrary $x_1 \in \Omega$.
\FOR{$t = 1,2,\cdots$}
\STATE $x_{t+\half} = \Pi_{\Omega} \InBrackets{ x_t - \eta_t g_{t-\half}  + \frac{1}{t+1}(x_1 -x_t)}$ 
\STATE Play $x_{t+\half}$ and receive feedback $g_{t+\half}$. 
\STATE $x_{t+1} = \Pi_{\Omega} \InBrackets{ x_{t} - \eta_{t} g_{t+\half} + \frac{1}{t+1}(x_1 -x_t)}$
\IF{$S_{t+1} := \sum_{s=2}^{t}\InNorms{g_{s+\half}-g_{s-\half}}^2 > 4500\pi D^2L^2$}
\STATE $\eta_{t+1} = \frac{1}{\sqrt{1+S_{t+1}}}$.
\ELSE  
\STATE $\eta_{t+1} = \eta_t$.
\ENDIF
\ENDFOR
\end{algorithmic}
\end{algorithm}
\begin{remark}
In the adversarial setting, $L$ and $D$ can be any positive real numbers. If all players use \Cref{alg}, $L$ should be an upper bound of the Lipschitz constant of the game, and $D$ should be an upper bound of the diameter $\InNorms{\boldx-\boldx'} \le D$ for $\boldx, \boldx' \in \X$. In other words, the players do not need to know exactly the environment that they are interacting with to carefully pick the learning rate. As long as they know an upper bound for the Lipschitz constant and the diameter of all games that they could potentially participate in, \Cref{alg} will successfully choose the appropriate learning rate for them. 
\end{remark}

\section{Worst-Case Regret in the Adversarial Environment}
\label{sec:regret}
In this section, we view \Cref{alg} as a single-agent online learning algorithm in the \emph{adversarial setting} where the loss functions $\{f_t\}_{t \in T}$ are chosen by an adversary. We show in \Cref{thm:regret} that \Cref{alg} achieves min-max optimal $\mathcal{O}(\sqrt{T})$ regret when the gradient feedback is bounded. It shows that \ref{AOG} is an optimal no-regret algorithm in the adversarial setting. Our result can also be construed in the game setting. Importantly, this interpretation does not require any assumptions regarding how other players select their actions, nor does it require the game to be monotone or smooth.%One could also interpret the result in the game setting, where we make no assumption on how the other players choose their actions, or the game being monotone or smooth.
\begin{theorem}[Optimal Regret Bound]\label{thm:regret}
    Consider online learning with action set $\Omega$, convex loss functions $(f_t: \Omega \rightarrow \R)_{t \in T}$ and gradient feedback $\{g_{t+\half}:= \nabla f_t(x_{t+\half})\}_{t \in [T]}$.   Let $G = \max_t \InNorms{g_{t+\half}}^2$ and suppose the action set $\Omega$ is bounded by $D$. The regret of Algorithm~\ref{alg} is bounded by $\mathcal{O}(D^2G\sqrt{T} + G^2)$.
\end{theorem}

We first establish a single-step regret inequality in \Cref{lem:single-step}. 
\begin{lemma}[Single-Step Regret Inequality]\label{lem:single-step} Suppose the action set $\Omega$ is bounded by $D$. For all $t\ge 1$ and any $x' \in \X$, the iterates of \ref{AOG} satisfies
\begin{align*}
&\InAngles{x_{t+\half} - x', g_{t+\half}} \le \frac{1}{2\eta_t}\InParentheses{\InNorms{x'-x_t}^2 - \InNorms{x' - x_{t+1}}^2}  \\
&\quad\quad +\eta_t \InNorms{g_{t+\half}-g_{t-\half}}^2 +  \frac{D^2}{\eta_t(t+1)}.
\end{align*}
\end{lemma}
The main idea behind \Cref{lem:single-step} is to view the update rule of \ref{AOG} as a standard update rule of \ref{OG} with modified gradients $g_{t-\half} - \frac{1}{\eta_t (t+1)}(x_1-x_t)$ and $g_{t+\half} - \frac{1}{\eta_t (t+1)}(x_1-x_t)$, which allows us to apply the classical analysis of \ref{OG}~\citep{rakhlin2013optimization}. Equipped with Lemma~\ref{lem:single-step}, we can bound the regret of \Cref{alg} even with adaptive size. We defer the proofs of \Cref{lem:single-step} and \Cref{thm:regret} to Appendix~\ref{app:single-step}.

\section{Last-Iterate Convergence Rate to a Nash Equilibrium in Monotone Games}\label{sec:last-iterate rate}
 In this section, we consider a multi-player learning setting where each player follows \ref{AOG} with constant step size in smooth monotone games: each player $i$ plays $x^i_{t+\half}$, receives gradient $V^i(\boldx_{t+\half})$, and updates
\begin{align*}
     x^i_{t+\half} &= \Pi_{\X^i} \InBrackets{ x^i_t - \eta V^i(\boldx_{t-\half})  + \frac{1}{t+1}(x^i_1 -x^i_t)}, \\
    x^i_{t+1} &= \Pi_{\X^i} \InBrackets{ x^i_{t} - \eta V^i(\boldx_{t+\half}) + \frac{1}{t+1}(x^i_1 -x^i_t)}.
\end{align*}

We show in \Cref{thm:last-iterate rate} that the trajectory of the action profile $(\boldx_{t+\half})_{t\in[T]}$ converges to Nash equilibrium in last-iterate with an $\mathcal{O}(\frac{1}{T})$ rate. Our convergence rate result matches the $\Omega(\frac{1}{T})$ lower bound by~\cite{yoon2021accelerated} and thus establishes that \ref{AOG} is doubly optimal. 

\begin{theorem}[Optimal Last-Iterate Convergence Rate]
\label{thm:last-iterate rate}
Let $\mathcal{G}=\{N,(\X^i)_{i \in [N]},(\ell^i)_{i\in[N]}\}$ be a $L$-smooth monotone game, where the diameter of $\X=\prod_{i\in [N]} \X^i$ is bounded by $D$. When all players employ \ref{AOG} with a constant step size $\eta \le \frac{1}{\sqrt{6}L}$ in $\mathcal{G}$, then for any $T \ge 2$, we have
\begin{itemize}
    \item $r^{tan}(\boldx_{T+\half}) \le \frac{55D}{\eta T}$;
    \item $\tgap(\boldx_{T+\half}) \le \gap(\boldx_{T+\half}) \le \frac{55D^2}{\eta T}$.
\end{itemize}
\end{theorem}

\begin{remark}
    In the same setup of \Cref{thm:last-iterate rate}, when the action set $\X$ is \emph{unbounded} (e.g., $\X = \R^{n}$), \ref{AOG} still enjoys last-iterate convergence with respect to the tangent residual. Let $\boldx_\star$ be any Nash equilibrium of the game. For any $T \ge 2$, we have 
    $ r^{tan}(\boldx_{T+\half}) \le \frac{1430 H}{\eta T}$,
    where $H = \max\{\InNorms{\boldx_1 - \boldx_{\star}},  r^{tan}(\boldx_1)\}$ is a constant that only depends on the choice of the initial point $\boldx_1$. We defer the proof to \Cref{app:unbounded domain}.
\end{remark}

\paragraph{A Sketch of the Proof.}  First, recall that the tangent residual provides upper bounds for both the gap function and the total gap function due to \cref{lem:gap bounded by residual}, so it suffices to prove a last-iterate convergence rate with respect to the tangent residual. For $\boldx \in \X$, its tangent residual is defined as $r^{tan}(\boldx) = \min_{c \in N_\X(\boldx)} \InNorms{V(\boldx) + c}$. The definition itself contains an optimization problem, thus is not explicit and difficult to directly work with. We relax the tangent residual by choosing an explicit $c \in N_\X(\boldx)$ as follows: for each player $i \in [N]$ and iteration $t \ge 2$, we define  \[c^i_t = \frac{x^i_{t-1} - \eta V^i(\boldx_{t-\half}) + \frac{1}{t}(x^i_1 -x^i_{t-1})- x^i_{t}}{\eta}.\] According to the update rule of \ref{AOG},  $c^i_t \in N_{\X^i}(x^i_t)$.  Define $c_t = (c_t^1, c_t^2, \cdots, c_t^N)$ and we have $c_t \in N_{\X}(\boldx_t)$.
Thus $r^{tan}(\boldx_t) = \min_{c \in N_\X(\boldx_t)} \InNorms{V(\boldx_t) + c} \le \InNorms{V(\boldx_t) + c_t}$. 

Using $\InNorms{V(\boldx_t) + c_t}$ as a proxy of the tangent residual $r^{tan}(\boldx_t)$, we construct a potential function of $P_t$ in the order of $\Theta(t^2 \cdot \InNorms{V(\boldx_t) + c_t}^2)$. Although the potential function might increase between consecutive iterates, we manage prove that in \cref{lem:potential} that the increment is sufficiently small: %\emph{approximately} non-increasing:
$P_{t+1} \le P_t + \mathcal{O}(\InNorms{V(\boldx_{t+1}) + c_{t+1}}^2)$ for any $t\ge 2$. Using the \emph{approximate monotonicity} of $P_t$, we derive the following inequality for the sequence  $(\InNorms{V(\boldx_t)+c_t}^2)_{t\ge 2}$
\begin{align*}
    \Theta(t^2 \cdot \InNorms{V(\boldx_t) + c_t}^2) \le \mathcal{O}(1) + \mathcal{O}(\sum_{s=2}^{t-1} \InNorms{V(\boldx_s) + c_s}^2).
\end{align*}
Based on the above inequality, we show in \cref{lem:ls-tangent} that $\InNorms{V(\boldx_t) + c_t}^2 = \mathcal{O}(\frac{1}{t^2})$ for any $t\ge 2$,  which implies $\mathcal{O}(\frac{1}{T})$ last-iterate convergence rate for $\boldx_t$. The final step is to relate the convergence on $\boldx_t$ to the convergence of the action profile $\boldx_{t+\half}$.

\subsection{Proof of \cref{thm:last-iterate rate}}
Some of the proofs are postponed to Appendix~\ref{app:last-iterate}. We also defer some auxiliary propositions to Appendix~\ref{app:auxiliary}.

\paragraph{Potential Function} 
We first formally define our potential function $P_t$: for $t \ge 2$, let $P_t$ be
\begin{align*}
    &\frac{t(t+1)}{2}\InParentheses{ \InNorms{\eta V(\boldx_t)+\eta c_t}^2 + \InNorms{\eta V(\boldx_t) - \eta V(\boldx_{t-\half})}^2} \\
    & \quad + t \InAngles{\eta V(\boldx_t) + \eta c_t, \boldx_t - \boldx_1}.
\end{align*}

We first provide an upper bound on $P_2$.
\begin{proposition}
\label{prop:P2}
In the same setup of \cref{thm:last-iterate rate}, $P_2 \le 9D^2$.
\end{proposition}

Now we present the main technical lemma of this section, where we show the potential function $P_t$ is approximately non-increasing. 
\begin{lemma}
\label{lem:potential}
    In the same setup of \cref{thm:last-iterate rate}, if we choose $\eta = \frac{\sqrt{q}}{L}$ for any $q \in (0, \frac{1}{4})$, then for all $t \ge 2$, \[P_{t+1} \le P_t + \frac{3q}{2(1-4q)}\InNorms{\eta V(\boldx_{t+1}) + \eta c_{t+1}}^2.\]
\end{lemma}
\begin{proof}
    We show $P_t - P_{t+1}$ minus a few non-negative terms is at least $- \frac{3q}{2(1-4q)}\InNorms{\eta V(\boldx_{t+1}) + \eta c_{t+1}}^2$. Here we present the list of non-negative terms that we use in the proof.

    \paragraph{Non-Negative Terms}
     Since the game is monotone, we have 
     \begin{align}\label{eq:potential-1}
        \InAngles{\eta V(\boldx_{t+1})-\eta V(\boldx_t), \boldx_{t+1}-\boldx_t} \ge 0.
    \end{align}
    Using the $L$-Lipschitzness of $V$ and the fact that $(\eta L)^2 \le q $, we have 
    \begin{gather}\label{eq:potential-2}
         q\InNorms{\boldx_{t+1} - \boldx_{t+\half}}^2 -  \InNorms{\eta V(\boldx_{t+1}) - \eta V(\boldx_{t+\half})}^2\ge 0.
    \end{gather}
    Since $c_t$ lies in the normal cone $N_{\X}(\boldx_{t})$ and $c_{t+1}$ lies in the normal cone $N_{\X}(\boldx_{t+1})$, by  the definition of normal cone we have 
    \begin{align}
        \InAngles{\eta c_{t+1}, \boldx_{t+1} - \boldx_t} \ge 0 \label{eq:potential-3} \\
        \InAngles{\eta c_t, \boldx_t - \boldx_{t+\half}} \ge 0 \label{eq:potential-4} \\
        \InAngles{\eta c_t, \boldx_t - \boldx_{t+1}} \ge 0 
        \label{eq:potential-5}
    \end{align}
    As $\boldx_t - \eta V(\boldx_{t-\half}) + \frac{1}{t+1}(\boldx_1 - \boldx_t)-\boldx_{t+\half} $ lies in the normal cone $N_\X(\boldx_{t+\half})$, we also have 
    \begin{align}\label{eq:potential-6}
        \InAngles{\boldx_t - \eta V(\boldx_{t-\half}) + \frac{\boldx_1 - \boldx_t}{t+1}-\boldx_{t+\half},\boldx_{t+\half} -\boldx_{t+1}} \ge 0.
    \end{align}

    \paragraph{Descent Identity}
    For convenience, we denote LHSI as ``left-hand side of inequality". We have the following identity by \cref{prop:identity}: 
    \begin{align*}
        &P_t - P_{t+1} - t(t+1) \cdot \text{LHSI}~\eqref{eq:potential-1} - \frac{t(t+1)}{4q} \cdot \text{LHSI}~\eqref{eq:potential-2} \nonumber \\
        &\quad - t(t+1) \cdot \text{LHSI}~\eqref{eq:potential-3} \\
        & \quad - \frac{t(t+1)}{2} \cdot \InParentheses{\text{LHSI}~\eqref{eq:potential-4} + \text{LHSI}~\eqref{eq:potential-5} + \text{LHSI}~\eqref{eq:potential-6}} \nonumber \\
        & = \frac{t(t+1)}{2} \InNorms{\frac{\boldx_{t+\half} -\boldx_{t+1}}{2} + \eta V(\boldx_t) - \eta V(\boldx_{t+\half})}^2 \nonumber \\
        & + \frac{t(t+1)}{2} \InNorms{\frac{\boldx_{t+\half} +\boldx_{t+1}}{2}- \boldx_t + \eta V(\boldx_t) + c_t - \frac{\boldx_1 -\boldx_t}{t+1}}^2\\
        &  + \underbrace{\frac{(1-4q)t-4q}{4q} (t+1) \InNorms{\eta V(\boldx_{t+\half}) - \eta V(\boldx_{t+1})}^2}_{\mathbf{I}}  \\
        &  + \underbrace{(t+1) \cdot \InAngles{\eta V(\boldx_{t+\half}) - \eta V(\boldx_{t+1}), \eta V(\boldx_{t+1}) + \eta c_{t+1}}.}_{\mathbf{II}}  
    \end{align*}
    Further using identity $\InNorms{a}^2 + \InAngles{a,b} = \InNorms{a+\frac{b}{2}}^2 - \frac{1}{4}\InNorms{b}^2 $, we can simplify the last two terms:
    \begin{align*}
        &\mathbf{I} + \mathbf{II} \\
        &= \InNorms{A (\eta V(\boldx_{t+\half}) - \eta V(\boldx_{t+1})) + B (\eta V(\boldx_{t+1}) + \eta c_{t+1})  }^2 \\
        & \quad - \frac{q(t+1)}{(1-4q)t-4q} \InNorms{\eta V(\boldx_{t+1}) + c_{t+1}}^2 \\
        & \ge - \frac{3q}{2(1-4q)} \InNorms{\eta V(\boldx_{t+1}) + c_{t+1}}^2,
    \end{align*}
    where $A = \sqrt{\frac{(1-4q)t-4q}{4q} (t+1) }$ , $B = \sqrt{\frac{q}{(1-4q)t-4q}(t+1)}$, and we use the fact that $\frac{t+1}{t} \le \frac{3}{2}$ for $t \ge 2$ in the last inequality. Combining the above two inequalities and the fact that we only add non-positive terms to $P_t - P_{t+1}$, we conclude that $P_{t+1} \le P_t + \frac{3q}{2(1-4q)}\InNorms{\eta V(\boldx_{t+1}) + c_{t+1}}^2$.
\end{proof}

Using the fact that the potential function $P_t$ is approximately non-increasing, we are able to use induction to show last-iterate convergence rate of the sequence $(x_t)_{t\ge 2}$.
\begin{lemma}
\label{lem:ls-tangent}
If $\X$ is bounded by $D$ and $\eta \in (0, \frac{1}{\sqrt{6}L})$, then we have  for all $T\ge 2$,
\begin{align*}
    \InNorms{V(\boldx_T) + c_T} \le \frac{13D}{\eta T}~~\text{and}~~\InNorms{ V(\boldx_T) -  V(\boldx_{T-\half})} \le \frac{13D}{\eta T}.
\end{align*}
\end{lemma}
\begin{proof}
    Let $\boldx_{\star}$ be a Nash equilibrium of $\mathcal{G}$. For any $t \ge 2$, we have
    \begingroup\makeatletter\def\f@size{9.5}\check@mathfonts
\def\maketag@@@#1{\hbox{\m@th\large\normalfont#1}}%
    \begin{align*}
        &P_t \\
        &=\frac{t(t+1)}{2}\InParentheses{ \InNorms{\eta V(\boldx_t)+\eta c_t}^2 + \InNorms{\eta V(\boldx_t) - \eta V(\boldx_{t-\half})}^2}  \\& \quad  + t \InAngles{\eta V(\boldx_t) + \eta c_t, \boldx_\star - \boldx_1} + t \InAngles{\eta V(\boldx_t) + \eta c_t, \boldx_t - \boldx_\star} \\
        &\ge \frac{t(t+1)}{2}\InParentheses{ \InNorms{\eta V(\boldx_t)+\eta c_t}^2 + \InNorms{\eta V(\boldx_t) - \eta V(\boldx_{t-\half})}^2}  \\
        &\quad + t \InAngles{\eta V(\boldx_t) + \eta c_t, \boldx_{\star} - x_1} \\
        &\ge \frac{t(t+1)}{4}\InParentheses{ \InNorms{\eta V(\boldx_t)+\eta c_t}^2 + 2\InNorms{\eta V(\boldx_t) - \eta V(\boldx_{t-\half})}^2} \\
        &\quad - \frac{t}{t+1} \InNorms{\boldx_{\star}-\boldx_1}^2 \\
        & \ge \frac{t(t+1)}{4}\InParentheses{ \InNorms{\eta V(\boldx_t)+\eta c_t}^2 + 2\InNorms{\eta V(\boldx_t) - \eta V(\boldx_{t-\half})}^2} \\
        &\quad -  \InNorms{\boldx_{\star}-\boldx_1}^2.
    \end{align*}\endgroup
    In the first inequality, we drop a positive term where  $\InAngles{V(\boldx_t), \boldx_t - \boldx_\star} \ge \InAngles{V(\boldx_\star), \boldx_t - \boldx_\star}  \ge 0$ since $\boldx_\star$ is Nash equilibrium, and $\InAngles{c_t, \boldx_t - \boldx_\star} \ge 0$ as $c_t \in N_\X(\boldx_t)$. In the second inequality, we apply inequality $\InAngles{a,b} \ge -\frac{\alpha}{4}\InNorms{a}^2 - \frac{1}{\alpha}\InNorms{b}^2$ with $a = \sqrt{t}\eta(V(\boldx_t) + c_t)$, $b = \sqrt{t} (\boldx_\star - \boldx_1)$, and $\alpha = t+1$; we use $ \frac{t}{t+1} \le 1$ in the last inequality.
    Combing the above inequality with Lemma~\ref{lem:potential} and Proposition~\ref{prop:P2}, we get for any $t\ge 2$,
    \begin{align*}
        &\frac{t(t+1)}{4}\InParentheses{ \InNorms{\eta V(\boldx_t)+\eta c_t}^2 + 2\InNorms{\eta V(\boldx_t) - \eta V(\boldx_{t-\half})}^2} \\
        &\le  \InNorms{\boldx_{\star}-\boldx_1}^2 + P_t \\
        &\le \InNorms{\boldx_{\star}-\boldx_1}^2 + P_2 + \frac{1}{3} \sum_{s=2}^{t-1} \InNorms{\eta V(\boldx_s) + \eta c_s}^2 \\
        & \le 10D^2 + \frac{1}{3} \sum_{s=2}^{t-1} \InNorms{\eta V(\boldx_s) + \eta c_s}^2 .
    \end{align*}
    By \cref{prop:sequence analysis}, we can conclude that for any $t \ge 2$, 
    \begin{align*}
        \InNorms{\eta V(\boldx_t)+\eta c_t}^2 + 2\InNorms{\eta V(\boldx_t) - \eta V(\boldx_{t-\half})}^2 \le \frac{160D^2}{t^2}.
    \end{align*}
    This completes the proof as $13^2 = 169 \ge 160$.
\end{proof}

Using the last-iterate convergence rate on $(\boldx_t)_{t\ge 2}$, we only need to bound the distance between $\boldx_t$ and $\boldx_{t+\half}$. 

\begin{lemma}
\label{lem:residual}
In the same setup of \cref{thm:last-iterate rate}, we have for any $t\ge2$, $ \InNorms{\boldx_{t+\half} - \boldx_t} \le \frac{27D}{t}$.
\end{lemma}

\paragraph{Proof of \cref{thm:last-iterate rate}} Given \cref{lem:ls-tangent} that proves the last-iterate convergence rate on the sequence $(\boldx_t)_{t\ge 2}$, and \cref{lem:residual} that upper bounds the distance between $\boldx_t$ and $\boldx_{t+\half}$,  we are now ready to prove the last-iterate convergence rate for $(\boldx_{t+\half})_{t\ge 2}$. 

Note that $\boldx_t - \eta V(\boldx_{t-\half}) + \frac{\boldx_1 - \boldx_t}{t+1} - \boldx_{t+\half} \in N_\X(\boldx_{t+\half})$, thus we can upper bound the tangent residual at $\boldx_{t+\half}$ by
\begin{align*}
    &r^{tan}(\boldx_{t+\half}) \\
    &= \frac{1}{\eta} \min_{c \in N_\X(\boldx_{t+\half})} \InNorms{\eta V(\boldx_{t+\half}) + c} \\
    &\le \frac{1}{\eta} \InNorms{ \eta V(\boldx_{t+\half}) +  \boldx_t - \eta V(\boldx_{t-\half}) + \frac{\boldx_1 - \boldx_t}{t+1} - \boldx_{t+\half}} \\
    &\le \InNorms{V(\boldx_{t}) -V(\boldx_{t-\half})} +    \frac{1+\eta L}{\eta}\InNorms{\boldx_{t+\half} - \boldx_t} + \frac{D}{\eta (t+1)}\\
    &\le \frac{13D}{\eta t} + \frac{\frac{3}{2} \cdot 27D}{\eta t} + \frac{D}{\eta (t+1)} \tag{\cref{lem:ls-tangent}, \ref{lem:residual} and $\eta L \le \frac{1}{2}$} \\
    & \le \frac{55D}{\eta t}, 
\end{align*}
where we use the triangle inequality and the $L$-Lipschitzness of $V$ in the second inequality. This completes the first part of \cref{thm:last-iterate rate}. The second part of \cref{thm:last-iterate rate} follows directly from the first part of \cref{thm:last-iterate rate} and \cref{lem:gap bounded by residual}.

\section{Dynamic Regret and Second-Order Gradient Variation}
Recent works on no-regret learning in games have provided near-optimal bounds for players' individual \emph{external} or \emph{swap} regret. In particular, \citet{daskalakis2021near-optimal,anagnostides2022near-optimal,anagnostides2022uncoupled} achieve logarithmic regret bounds for general-sum games, and the bound can be sharpen to $O(1)$ if the games are monotone~\citep{hsieh2021adaptive}. However, \emph{dynamic regret} is a much stronger concept, which is impossible to achieve in the single-agent adversarial setting and tightly relates to the concept of last-iterate convergence in game settings. For example, the $\mathcal{O}(\frac{1}{\sqrt{T}})$ last-iterate convergence rate of \ref{OG} implies a $\mathcal{O}(\sqrt{T})$ individual dynamic regret bound in monotone games.  To the best of our knowledge, $\mathcal{O}(\sqrt{T})$ is the best bound for dynamic regret even in two-player zero-sum games.

We significantly improve the bound and show that the individual dynamic regret is at most $\mathcal{O}(\log T)$ if each player employs \ref{AOG} in monotone games\footnote{\citet{anagnostides2023convergence} shows an $\mathcal{O}(\log T)$ regret bound for two-player zero-sum games but under a stronger two-point feedback model. In their model, the algorithm is allowed to query the payoff vector/gradient at two different strategies in each iteration, while the regret is calculated with respect to only the first queried strategy.}. This is made possible by the fast $\mathcal{O}(\frac{1}{T})$ last-iterate convergence rate of \ref{AOG}.
\begin{theorem}[Individual Dynamic Regret Bound]
\label{thm:dynamic-regret}
    In the same setup of \cref{thm:last-iterate rate}, for any $i \in [N]$ and $T\ge 2$, 
    \begin{align*}
        \mathrm{DynamicReg}^i(T) \le \mathcal{O}(\log T).
    \end{align*}
\end{theorem}
\begin{proof}
    By the definition of dynamic regret and  total gap function, for any $T \ge 2$, we have
    \begin{align*}
        &\mathrm{DynamicReg}^i(T) = \sum_{t=1}^T \InParentheses{ \ell^i(\boldx_{t+\half}) - \min_{x' \in \X^i} \ell^i(x',x^{-i}_{t+\half})}\\
        &\le \mathcal{O}(1) + \sum_{t=2}^T \tgap(\boldx_{t+\half}) \le \sum_{t=2}^T \mathcal{O}(\frac{1}{t}) = \mathcal{O}(\log T). \qedhere
    \end{align*}
\end{proof}

Last-iterate convergence rate of \ref{AOG} also implies each player's bounded second-order gradient variation. We defer the proof of \cref{thm:second-order gradient variation} to Appendix~\ref{app:second-order}.
\begin{theorem}[Bounded Second-Order Gradient Variation]
\label{thm:second-order gradient variation}
In the same setup of \cref{thm:last-iterate rate} but with $\eta = \frac{1}{3L}$,  for any player $i$ and time $t\ge 2$, we have $S^i_T \le 4500\pi D^2 L^2$.
\end{theorem}
Bounded second-order gradient variation guarantees when each player employs \cref{alg} with the step-size adaptation procedure, they will always use constant step size. Combining \cref{thm:regret}, \cref{thm:last-iterate rate}, and \cref{thm:second-order gradient variation}, we conclude that \cref{alg} is doubly optimal. 
\begin{theorem}
\label{thm:adaptive-optimal}
    \cref{alg} automatically adapts to the environment and achieves $\mathcal{O}(\sqrt{T})$ regret in the adversarial setting and $\mathcal{O}(\frac{1}{T})$ last-iterate convergence rate in smooth monotone games.
\end{theorem}

\section{Illustrative Experiments}
\begin{figure}[ht]
    \centering
    \begin{subfigure}
        \centering
        \includegraphics[width=0.4\textwidth]{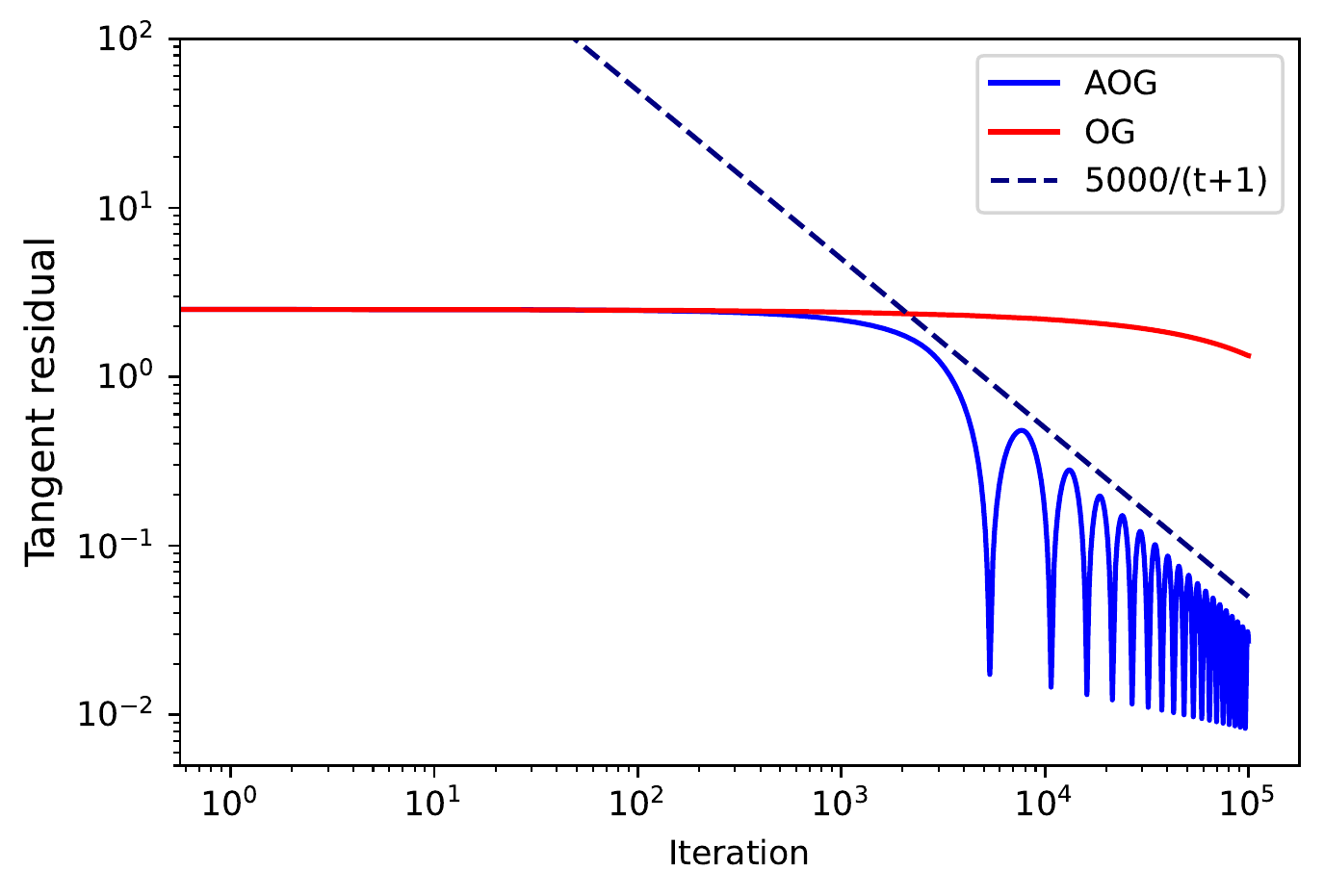} % first figure itself
    \end{subfigure}\hfill
    \begin{subfigure}
        \centering
        \includegraphics[width=0.4\textwidth]{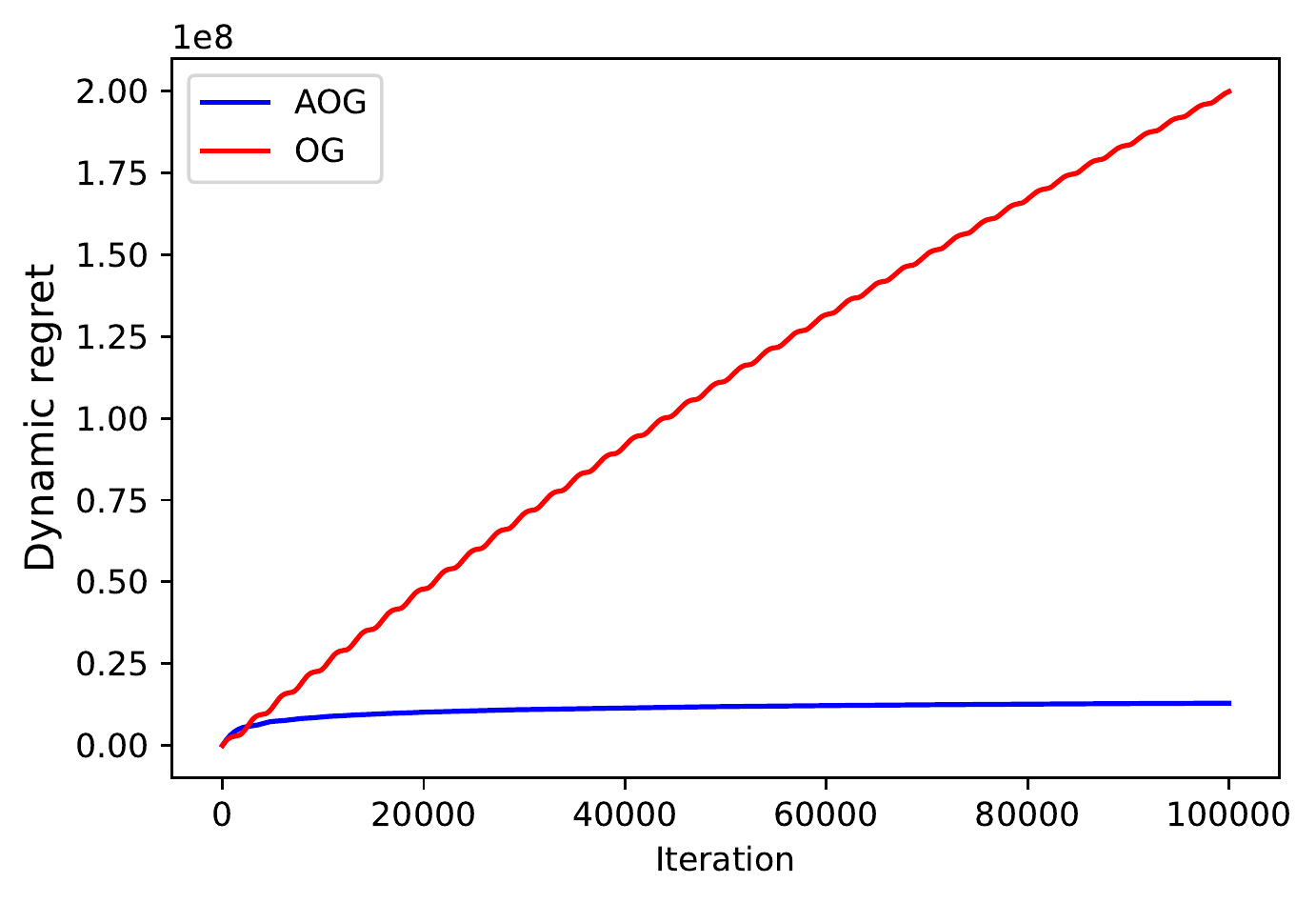} % second figure itself
    \end{subfigure}
    \caption{Numerical Results of \ref{AOG} and \ref{OG}.  }
    \label{fig:experiment}
\end{figure}

In this section, we numerically verify our theoretical results through \cref{ex:min-max}. Let $A \in \R^{n\times n}$, $b, h \in \R^n$, and $\X, \Y \subseteq \R^n$, and $f: \X \times \Y \rightarrow R$ be of the form $f(x,y) = \frac{1}{2}x^\top Hx - h^\top x - \InAngles{Ax -b, y}$~\cite{ouyang2021lower}. We consider a convex-concave min-max optimization problem $\min_{x \in \X} \max_{y \in \Y} f(x,y)$, which is also a two-player zero-sum game with $f^1 = -f^2 = f$.  Details of the choices of $H, A, b, h, \X, \Y$ and step size $\eta$ are deferred to Appendix~\ref{app:experiment}. 

The numerical result is shown in \cref{fig:experiment}. We use $\boldsymbol{z}$ to denote $(x,y)$. When players use \ref{AOG}, the tangent residual of players' action profile $r^{tan}(\boldsymbol{z}_{t+\half})$ decreases at a rate of $\mathcal{O}(\frac{1}{T})$, and corroborates our theoretical results (\cref{thm:last-iterate rate}). Moreover, \ref{AOG} significantly outperforms \ref{OG} in terms of both the last-iterate convergence rate and the individual dynamic regret.

\section{Related Work}
\paragraph{Last-Iterate Convergence of No-regret learning in Games} 
There is a vast literature on no-regret learning in games. For strongly monotone games, linear last-iterate convergence rate is known~\citep{tseng_linear_1995,liang_interaction_2019,mokhtari2020convergence,zhou_convergence_2020}. Even under bandit feedback or noisy gradient feedback, optimal sub-linear last-iterate convergence rate is achieved by no-regret learning algorithms for strongly monotone games~\citep{lin_doubly_2022, jordan2022adaptive}.

Obtaining last-iterate convergence rate to Nash equilibria beyond strongly monotone games received extensive attention recently. \citet{daskalakis2018limit} proved asymptotic convergence of the optimistic gradient \eqref{OG} algorithm in zero-sum games. Asymptotic convergence was also achieved in variationally stable games~\citep{zhou_mirror_2017,zhou2017countering, mertikopoulos2019learning,hsieh2021adaptive} even with noisy feedback~\citep{hsieh2022no}. Finite time $\mathcal{O}(\frac{1}{\sqrt{T}})$ convergence was shown for unconstrained cocoercive games~\citep{lin_finite-time_2020} and unconstrained monotone games~\citep{golowich2020tight}. For bilinear games over polytopes, \cite{wei2021linear} show linear convergence rate of \ref{OG} but this rate depends on a problem constant $c$ which can be arbitrarily large. Recently, \citet{cai2022finite} proved a tight $\mathcal{O}(\frac{1}{\sqrt{T}})$ last-iterate convergence rate of \ref{OG} and the extragradient (EG) algortihm for constrained monotone games, matching the lower bound of p-SCIL algorithms by~\citet{golowich2020tight}. We remark that for general gradient-based algorithms, the lower bound is $\Omega(\frac{1}{T})$~\citep{ouyang2021lower,yoon2021accelerated}. 

\vspace{-.1in}
\paragraph{Regret Minimization in Games}
There is a large collection of works on minimizing  individual regret in games, from early results in two-player zero-sum games~\citep{daskalakis2011near,kangarshahi2018let} to more recent works on general-sum games~\citep{syrgkanis2015fast,chen2020hedging, daskalakis2021near-optimal,anagnostides2022near-optimal, anagnostides2022uncoupled}. Among them, \cite{daskalakis2021near-optimal,anagnostides2022near-optimal, anagnostides2022uncoupled} achieves $\mathcal{O}(\log T)$ regret for general-sum games and \cite{hsieh2021adaptive} achieves $\mathcal{O}(1)$ regret for variationally stable games. Little is known, however, for the stronger notion of dynamic regret except for $\mathcal{O}(\sqrt{T})$ bound of \ref{OG} in monotone games~\citep{cai2022finite}.

\vspace{-.1in}
\paragraph{Learning in Repeated Games and Evolutionary Game Theory} Agents in our model could be interpreted as different populations rather than individuals, where the mixed strategy describes the prevalence of each of the pure strategies in the population. Under this interpretation, we no longer have the same players playing the same game every day. Instead, in each round of the repeated game, individuals from one population play the game with other individuals drawn randomly from other populations. Such interpretation has wide application in evolutionary game theory, where repeated games are used to model evolution (see the monograph of \citep{weibull1997evolutionary} and the references therein for more details).

\vspace{-.15in}
\section{Conclusion and Discussion}
In this paper, we propose the first doubly optimal online learning algorithm, the accelerated optimistic gradient \eqref{AOG} algorithm, which achieves optimal $\mathcal{O}(\sqrt{T})$ regret bound in the adversarial setting and optimal $\mathcal{O}(\frac{1}{T})$ last-iterate convergence rate in smooth monotone games. Extending our results in settings where players only receive noisy gradient or even bandit feedback is an interesting and challenging future direction. Finally, We significantly improve the state-of-the-art  upper bound of the individual dynamic regret from $\mathcal{O}(\sqrt{T})$ to $\mathcal{O}(\log{T})$. We believe that understanding the optimal individual dynamic regret is an interesting open question for learning in monotone games.

\vspace{-.1in}
\paragraph{Open Question:}What is the optimal individual dynamic regret achievable in smooth monotone games using no-regret learning algorithms?
\vspace{-.1in}
\subsubsection*{Acknowledgements}
Yang Cai is supported by a Sloan Foundation Research Fellowship and the NSF Award CCF-1942583 (CAREER). We thank the anonymous reviewers for their constructive comments. 
%\newpage
%\input{ICML/checklist.tex}
\bibliography{references, ref}

\begin{thebibliography}{59}
\providecommand{\natexlab}[1]{#1}
\providecommand{\url}[1]{\texttt{#1}}
\expandafter\ifx\csname urlstyle\endcsname\relax
  \providecommand{\doi}[1]{doi: #1}\else
  \providecommand{\doi}{doi: \begingroup \urlstyle{rm}\Url}\fi

\bibitem[Anagnostides et~al.(2022{\natexlab{a}})Anagnostides, Daskalakis,
  Farina, Fishelson, Golowich, and Sandholm]{anagnostides2022near-optimal}
Anagnostides, I., Daskalakis, C., Farina, G., Fishelson, M., Golowich, N., and
  Sandholm, T.
\newblock Near-optimal no-regret learning for correlated equilibria in
  multi-player general-sum games.
\newblock In \emph{Proceedings of the 54th Annual ACM SIGACT Symposium on
  Theory of Computing (STOC)}, 2022{\natexlab{a}}.

\bibitem[Anagnostides et~al.(2022{\natexlab{b}})Anagnostides, Farina, Kroer,
  Lee, Luo, and Sandholm]{anagnostides2022uncoupled}
Anagnostides, I., Farina, G., Kroer, C., Lee, C.-W., Luo, H., and Sandholm, T.
\newblock Uncoupled learning dynamics with $o(\log t)$ swap regret in
  multiplayer games.
\newblock In \emph{Advances in Neural Information Processing Systems
  (NeurIPS)}, 2022{\natexlab{b}}.

\bibitem[Anagnostides et~al.(2023)Anagnostides, Panageas, Farina, and
  Sandholm]{anagnostides2023convergence}
Anagnostides, I., Panageas, I., Farina, G., and Sandholm, T.
\newblock On the convergence of no-regret learning dynamics in time-varying
  games.
\newblock \emph{arXiv preprint arXiv:2301.11241}, 2023.

\bibitem[Arjovsky et~al.(2017)Arjovsky, Chintala, and
  Bottou]{arjovsky_wasserstein_2017}
Arjovsky, M., Chintala, S., and Bottou, L.
\newblock Wasserstein {Generative} {Adversarial} {Networks}.
\newblock In \emph{Proceedings of the 34th {International} {Conference} on
  {Machine} {Learning}}, pp.\  214--223. PMLR, July 2017.
\newblock URL \url{https://proceedings.mlr.press/v70/arjovsky17a.html}.
\newblock ISSN: 2640-3498.

\bibitem[Bregman \& Fokin(1987)Bregman and Fokin]{bregman_methods_1987}
Bregman, L. and Fokin, I.
\newblock Methods of {Determining} {Equilibrium} {Situations} in {Zero}-{Sum}
  {Polymatrix} {Games}.
\newblock \emph{Optimizatsia}, 40\penalty0 (57):\penalty0 70--82, 1987.

\bibitem[Cai \& Daskalakis(2011)Cai and Daskalakis]{cai2011minmax}
Cai, Y. and Daskalakis, C.
\newblock On minmax theorems for multiplayer games.
\newblock In \emph{Proceedings of the twenty-second annual ACM-SIAM symposium
  on Discrete algorithms (SODA)}, 2011.

\bibitem[Cai \& Zheng(2023)Cai and Zheng]{cai2023accelerated}
Cai, Y. and Zheng, W.
\newblock Accelerated single-call methods for constrained min-max optimization.
\newblock \emph{International Conference on Learning Representations (ICLR)},
  2023.
\newblock To appear.

\bibitem[Cai et~al.(2016)Cai, Candogan, Daskalakis, and
  Papadimitriou]{cai_zero-sum_2016}
Cai, Y., Candogan, O., Daskalakis, C., and Papadimitriou, C.
\newblock Zero-{Sum} {Polymatrix} {Games}: {A} {Generalization} of {Minmax}.
\newblock \emph{Mathematics of Operations Research}, 41\penalty0 (2):\penalty0
  648--655, May 2016.
\newblock ISSN 0364-765X.
\newblock \doi{10.1287/moor.2015.0745}.
\newblock URL \url{https://pubsonline.informs.org/doi/10.1287/moor.2015.0745}.
\newblock Publisher: INFORMS.

\bibitem[Cai et~al.(2022{\natexlab{a}})Cai, Oikonomou, and
  Zheng]{cai2022accelerated}
Cai, Y., Oikonomou, A., and Zheng, W.
\newblock Accelerated algorithms for monotone inclusion and constrained
  nonconvex-nonconcave min-max optimization.
\newblock \emph{arXiv preprint arXiv:2206.05248}, 2022{\natexlab{a}}.

\bibitem[Cai et~al.(2022{\natexlab{b}})Cai, Oikonomou, and
  Zheng]{cai2022finite}
Cai, Y., Oikonomou, A., and Zheng, W.
\newblock Finite-time last-iterate convergence for learning in multi-player
  games.
\newblock In \emph{Advances in Neural Information Processing Systems
  (NeurIPS)}, 2022{\natexlab{b}}.

\bibitem[Cesa-Bianchi \& Lugosi(2006)Cesa-Bianchi and
  Lugosi]{cesa-bianchi_prediction_2006}
Cesa-Bianchi, N. and Lugosi, G.
\newblock \emph{Prediction, {Learning}, and {Games}}.
\newblock Cambridge University Press, 2006.

\bibitem[Chen \& Peng(2020)Chen and Peng]{chen2020hedging}
Chen, X. and Peng, B.
\newblock Hedging in games: Faster convergence of external and swap regrets.
\newblock \emph{Advances in Neural Information Processing Systems (NeurIPS)},
  33:\penalty0 18990--18999, 2020.

\bibitem[Daskalakis \& Panageas(2018)Daskalakis and
  Panageas]{daskalakis2018limit}
Daskalakis, C. and Panageas, I.
\newblock The limit points of (optimistic) gradient descent in min-max
  optimization.
\newblock \emph{Advances in neural information processing systems (NeurIPS)},
  2018.

\bibitem[Daskalakis \& Panageas(2019)Daskalakis and
  Panageas]{daskalakis2019last}
Daskalakis, C. and Panageas, I.
\newblock Last-iterate convergence: Zero-sum games and constrained min-max
  optimization.
\newblock In \emph{10th Innovations in Theoretical Computer Science Conference
  (ITCS)}, 2019.

\bibitem[Daskalakis \& Papadimitriou(2009)Daskalakis and
  Papadimitriou]{daskalakis_network_2009}
Daskalakis, C. and Papadimitriou, C.~H.
\newblock On a {Network} {Generalization} of the {Minmax} {Theorem}.
\newblock In \emph{Proceedings of the 36th {Internatilonal} {Collogquium} on
  {Automata}, {Languages} and {Programming}: {Part} {II}}, {ICALP} '09, pp.\
  423--434, Berlin, Heidelberg, July 2009. Springer-Verlag.
\newblock ISBN 978-3-642-02929-5.
\newblock \doi{10.1007/978-3-642-02930-1_35}.
\newblock URL \url{https://doi.org/10.1007/978-3-642-02930-1_35}.

\bibitem[Daskalakis et~al.(2011)Daskalakis, Deckelbaum, and
  Kim]{daskalakis2011near}
Daskalakis, C., Deckelbaum, A., and Kim, A.
\newblock Near-optimal no-regret algorithms for zero-sum games.
\newblock In \emph{Proceedings of the twenty-second annual ACM-SIAM symposium
  on Discrete Algorithms}, pp.\  235--254. SIAM, 2011.

\bibitem[Daskalakis et~al.(2018)Daskalakis, Ilyas, Syrgkanis, and
  Zeng]{daskalakis2018training}
Daskalakis, C., Ilyas, A., Syrgkanis, V., and Zeng, H.
\newblock Training gans with optimism.
\newblock In \emph{International Conference on Learning Representations
  (ICLR)}, 2018.

\bibitem[Daskalakis et~al.(2021)Daskalakis, Fishelson, and
  Golowich]{daskalakis2021near-optimal}
Daskalakis, C., Fishelson, M., and Golowich, N.
\newblock Near-optimal no-regret learning in general games.
\newblock \emph{Advances in Neural Information Processing Systems (NeurIPS)},
  2021.

\bibitem[Diakonikolas(2020)]{Diakonikolas2020halpern}
Diakonikolas, J.
\newblock Halpern iteration for near-optimal and parameter-free monotone
  inclusion and strong solutions to variational inequalities.
\newblock In \emph{Conference on Learning Theory (COLT)}, 2020.

\bibitem[Even-Dar et~al.(2009)Even-Dar, Mansour, and
  Nadav]{even-dar_convergence_2009}
Even-Dar, E., Mansour, Y., and Nadav, U.
\newblock On the convergence of regret minimization dynamics in concave games.
\newblock In \emph{Proceedings of the forty-first annual {ACM} symposium on
  {Theory} of computing}, pp.\  523--532, 2009.

\bibitem[Fudenberg et~al.(1998)Fudenberg, Drew, Levine, and
  Levine]{fudenberg1998theory}
Fudenberg, D., Drew, F., Levine, D.~K., and Levine, D.~K.
\newblock \emph{The theory of learning in games}, volume~2.
\newblock MIT press, 1998.

\bibitem[Golowich et~al.(2020{\natexlab{a}})Golowich, Pattathil, and
  Daskalakis]{golowich2020tight}
Golowich, N., Pattathil, S., and Daskalakis, C.
\newblock Tight last-iterate convergence rates for no-regret learning in
  multi-player games.
\newblock \emph{Advances in neural information processing systems (NeurIPS)},
  2020{\natexlab{a}}.

\bibitem[Golowich et~al.(2020{\natexlab{b}})Golowich, Pattathil, Daskalakis,
  and Ozdaglar]{golowich2020last}
Golowich, N., Pattathil, S., Daskalakis, C., and Ozdaglar, A.
\newblock Last iterate is slower than averaged iterate in smooth convex-concave
  saddle point problems.
\newblock In \emph{Conference on Learning Theory (COLT)}, 2020{\natexlab{b}}.

\bibitem[Golowich et~al.(2020{\natexlab{c}})Golowich, Pattathil, Daskalakis,
  and Ozdaglar]{golowich_last_2020}
Golowich, N., Pattathil, S., Daskalakis, C., and Ozdaglar, A.
\newblock Last {Iterate} is {Slower} than {Averaged} {Iterate} in {Smooth}
  {Convex}-{Concave} {Saddle} {Point} {Problems}.
\newblock \emph{arXiv:2002.00057 [cs, math, stat]}, July 2020{\natexlab{c}}.
\newblock URL \url{http://arxiv.org/abs/2002.00057}.
\newblock arXiv: 2002.00057.

\bibitem[Halpern(1967)]{halpern1967fixed}
Halpern, B.
\newblock Fixed points of nonexpanding maps.
\newblock \emph{Bulletin of the American Mathematical Society}, 73\penalty0
  (6):\penalty0 957--961, 1967.

\bibitem[Hsieh et~al.(2019)Hsieh, Iutzeler, Malick, and
  Mertikopoulos]{hsieh2019convergence}
Hsieh, Y.-G., Iutzeler, F., Malick, J., and Mertikopoulos, P.
\newblock On the convergence of single-call stochastic extra-gradient methods.
\newblock \emph{Advances in Neural Information Processing Systems (NeurIPS)},
  2019.

\bibitem[Hsieh et~al.(2021)Hsieh, Antonakopoulos, and
  Mertikopoulos]{hsieh2021adaptive}
Hsieh, Y.-G., Antonakopoulos, K., and Mertikopoulos, P.
\newblock Adaptive learning in continuous games: Optimal regret bounds and
  convergence to nash equilibrium.
\newblock In \emph{Conference on Learning Theory}, pp.\  2388--2422. PMLR,
  2021.

\bibitem[Hsieh et~al.(2022)Hsieh, Antonakopoulos, Cevher, and
  Mertikopoulos]{hsieh2022no}
Hsieh, Y.-G., Antonakopoulos, K., Cevher, V., and Mertikopoulos, P.
\newblock No-regret learning in games with noisy feedback: Faster rates and
  adaptivity via learning rate separation.
\newblock In \emph{International Conference on Neural Information Processing
  Systems (NeurIPS)}, 2022.

\bibitem[Jordan et~al.(2022)Jordan, Lin, and Zhou]{jordan2022adaptive}
Jordan, M.~I., Lin, T., and Zhou, Z.
\newblock Adaptive, doubly optimal no-regret learning in games with gradient
  feedback.
\newblock \emph{Available at SSRN}, 2022.

\bibitem[Kangarshahi et~al.(2018)Kangarshahi, Hsieh, Sahin, and
  Cevher]{kangarshahi2018let}
Kangarshahi, E.~A., Hsieh, Y.-P., Sahin, M.~F., and Cevher, V.
\newblock Let’s be honest: An optimal no-regret framework for zero-sum games.
\newblock In \emph{International Conference on Machine Learning (ICML)}, 2018.

\bibitem[Korpelevich(1976)]{korpelevich_extragradient_1976}
Korpelevich, G.~M.
\newblock The extragradient method for finding saddle points and other
  problems.
\newblock \emph{Matecon}, 12:\penalty0 747--756, 1976.
\newblock URL \url{https://ci.nii.ac.jp/naid/10017556617/}.

\bibitem[Lee \& Kim(2021)Lee and Kim]{Lee2021fast}
Lee, S. and Kim, D.
\newblock Fast extra gradient methods for smooth structured
  nonconvex-nonconcave minimax problems.
\newblock In \emph{Annual Conference on Neural Information Processing Systems
  (NeurIPS)}, 2021.

\bibitem[Lei et~al.(2021)Lei, Nagarajan, Panageas, et~al.]{lei2021last}
Lei, Q., Nagarajan, S.~G., Panageas, I., et~al.
\newblock Last iterate convergence in no-regret learning: constrained min-max
  optimization for convex-concave landscapes.
\newblock In \emph{International Conference on Artificial Intelligence and
  Statistics}, 2021.

\bibitem[Liang \& Stokes(2019)Liang and Stokes]{liang_interaction_2019}
Liang, T. and Stokes, J.
\newblock Interaction matters: {A} note on non-asymptotic local convergence of
  generative adversarial networks.
\newblock In \emph{The 22nd {International} {Conference} on {Artificial}
  {Intelligence} and {Statistics}}, pp.\  907--915. PMLR, 2019.

\bibitem[Lin et~al.(2020)Lin, Zhou, Mertikopoulos, and
  Jordan]{lin_finite-time_2020}
Lin, T., Zhou, Z., Mertikopoulos, P., and Jordan, M.
\newblock Finite-{Time} {Last}-{Iterate} {Convergence} for {Multi}-{Agent}
  {Learning} in {Games}.
\newblock In \emph{Proceedings of the 37th {International} {Conference} on
  {Machine} {Learning}}, pp.\  6161--6171. PMLR, November 2020.
\newblock URL \url{https://proceedings.mlr.press/v119/lin20h.html}.
\newblock ISSN: 2640-3498.

\bibitem[Lin et~al.(2022)Lin, Zhou, Ba, and Zhang]{lin_doubly_2022}
Lin, T., Zhou, Z., Ba, W., and Zhang, J.
\newblock Doubly {Optimal} {No}-{Regret} {Online} {Learning} in {Strongly}
  {Monotone} {Games} with {Bandit} {Feedback}, July 2022.
\newblock URL \url{http://arxiv.org/abs/2112.02856}.
\newblock arXiv:2112.02856 [cs, math].

\bibitem[Mertikopoulos \& Zhou(2019)Mertikopoulos and
  Zhou]{mertikopoulos2019learning}
Mertikopoulos, P. and Zhou, Z.
\newblock Learning in games with continuous action sets and unknown payoff
  functions.
\newblock \emph{Mathematical Programming}, 173:\penalty0 465--507, 2019.

\bibitem[Mertikopoulos et~al.(2018)Mertikopoulos, Papadimitriou, and
  Piliouras]{mertikopoulos_cycles_2018}
Mertikopoulos, P., Papadimitriou, C., and Piliouras, G.
\newblock Cycles in adversarial regularized learning.
\newblock In \emph{Proceedings of the {Twenty}-{Ninth} {Annual} {ACM}-{SIAM}
  {Symposium} on {Discrete} {Algorithms}}, pp.\  2703--2717. SIAM, 2018.

\bibitem[Mokhtari et~al.(2020{\natexlab{a}})Mokhtari, Ozdaglar, and
  Pattathil]{mokhtari2020unified}
Mokhtari, A., Ozdaglar, A., and Pattathil, S.
\newblock A unified analysis of extra-gradient and optimistic gradient methods
  for saddle point problems: Proximal point approach.
\newblock In \emph{International Conference on Artificial Intelligence and
  Statistics (AISTATS)}, 2020{\natexlab{a}}.

\bibitem[Mokhtari et~al.(2020{\natexlab{b}})Mokhtari, Ozdaglar, and
  Pattathil]{mokhtari2020convergence}
Mokhtari, A., Ozdaglar, A.~E., and Pattathil, S.
\newblock Convergence rate of $\mathcal{O}(1/k)$ for optimistic gradient and
  extragradient methods in smooth convex-concave saddle point problems.
\newblock \emph{SIAM Journal on Optimization}, 30\penalty0 (4):\penalty0
  3230--3251, 2020{\natexlab{b}}.

\bibitem[Ouyang \& Xu(2021)Ouyang and Xu]{ouyang2021lower}
Ouyang, Y. and Xu, Y.
\newblock Lower complexity bounds of first-order methods for convex-concave
  bilinear saddle-point problems.
\newblock \emph{Mathematical Programming}, 185\penalty0 (1):\penalty0 1--35,
  2021.

\bibitem[Perolat et~al.(2022)Perolat, De~Vylder, Hennes, Tarassov, Strub,
  de~Boer, Muller, Connor, Burch, Anthony, et~al.]{perolat2022mastering}
Perolat, J., De~Vylder, B., Hennes, D., Tarassov, E., Strub, F., de~Boer, V.,
  Muller, P., Connor, J.~T., Burch, N., Anthony, T., et~al.
\newblock Mastering the game of stratego with model-free multiagent
  reinforcement learning.
\newblock \emph{Science}, 378\penalty0 (6623):\penalty0 990--996, 2022.

\bibitem[Popov(1980)]{popov_modification_1980}
Popov, L.~D.
\newblock A modification of the {Arrow}-{Hurwicz} method for search of saddle
  points.
\newblock \emph{Mathematical notes of the Academy of Sciences of the USSR},
  28\penalty0 (5):\penalty0 845--848, 1980.
\newblock Publisher: Springer.

\bibitem[Rakhlin \& Sridharan(2013)Rakhlin and
  Sridharan]{rakhlin2013optimization}
Rakhlin, S. and Sridharan, K.
\newblock Optimization, learning, and games with predictable sequences.
\newblock \emph{Advances in Neural Information Processing Systems}, 2013.

\bibitem[Rosen(1965)]{rosen_existence_1965}
Rosen, J.~B.
\newblock Existence and {Uniqueness} of {Equilibrium} {Points} for {Concave}
  {N}-{Person} {Games}.
\newblock \emph{Econometrica}, 33\penalty0 (3):\penalty0 520--534, 1965.
\newblock ISSN 0012-9682.
\newblock \doi{10.2307/1911749}.
\newblock URL \url{https://www.jstor.org/stable/1911749}.
\newblock Publisher: [Wiley, Econometric Society].

\bibitem[Silver et~al.(2017)Silver, Schrittwieser, Simonyan, Antonoglou, Huang,
  Guez, Hubert, Baker, Lai, Bolton, et~al.]{silver2017mastering}
Silver, D., Schrittwieser, J., Simonyan, K., Antonoglou, I., Huang, A., Guez,
  A., Hubert, T., Baker, L., Lai, M., Bolton, A., et~al.
\newblock Mastering the game of go without human knowledge.
\newblock \emph{nature}, 550\penalty0 (7676):\penalty0 354--359, 2017.

\bibitem[Sorin(2002)]{sorin2002first}
Sorin, S.
\newblock \emph{A first course on zero-sum repeated games}, volume~37.
\newblock Springer Science \& Business Media, 2002.

\bibitem[Syrgkanis et~al.(2015)Syrgkanis, Agarwal, Luo, and
  Schapire]{syrgkanis2015fast}
Syrgkanis, V., Agarwal, A., Luo, H., and Schapire, R.~E.
\newblock Fast convergence of regularized learning in games.
\newblock \emph{Advances in Neural Information Processing Systems (NeurIPS)},
  2015.

\bibitem[Tran-Dinh(2022)]{tran2022connection}
Tran-Dinh, Q.
\newblock The connection between nesterov's accelerated methods and halpern
  fixed-point iterations.
\newblock \emph{arXiv preprint arXiv:2203.04869}, 2022.

\bibitem[Tseng(1995)]{tseng_linear_1995}
Tseng, P.
\newblock On linear convergence of iterative methods for the variational
  inequality problem.
\newblock \emph{Journal of Computational and Applied Mathematics}, 60\penalty0
  (1):\penalty0 237--252, June 1995.
\newblock ISSN 0377-0427.
\newblock \doi{10.1016/0377-0427(94)00094-H}.
\newblock URL
  \url{https://www.sciencedirect.com/science/article/pii/037704279400094H}.

\bibitem[Viossat \& Zapechelnyuk(2013)Viossat and
  Zapechelnyuk]{viossat_no-regret_2013}
Viossat, Y. and Zapechelnyuk, A.
\newblock No-regret {Dynamics} and {Fictitious} {Play}.
\newblock \emph{Journal of Economic Theory}, 148\penalty0 (2):\penalty0
  825--842, March 2013.
\newblock ISSN 00220531.
\newblock \doi{10.1016/j.jet.2012.07.003}.
\newblock URL \url{http://arxiv.org/abs/1207.0660}.
\newblock arXiv: 1207.0660.

\bibitem[Wei et~al.(2021)Wei, Lee, Zhang, and Luo]{wei2021linear}
Wei, C.-Y., Lee, C.-W., Zhang, M., and Luo, H.
\newblock Linear last-iterate convergence in constrained saddle-point
  optimization.
\newblock In \emph{International Conference on Learning Representations
  (ICLR)}, 2021.

\bibitem[Weibull(1997)]{weibull1997evolutionary}
Weibull, J.~W.
\newblock \emph{Evolutionary game theory}.
\newblock MIT press, 1997.

\bibitem[Yoon \& Ryu(2021)Yoon and Ryu]{yoon2021accelerated}
Yoon, T. and Ryu, E.~K.
\newblock Accelerated algorithms for smooth convex-concave minimax problems
  with $\mathcal{O}(1/k^2)$ rate on squared gradient norm.
\newblock In \emph{International Conference on Machine Learning (ICML)}, pp.\
  12098--12109. PMLR, 2021.

\bibitem[Zhou et~al.(2017{\natexlab{a}})Zhou, Mertikopoulos, Bambos, Glynn, and
  Tomlin]{zhou2017countering}
Zhou, Z., Mertikopoulos, P., Bambos, N., Glynn, P.~W., and Tomlin, C.
\newblock Countering feedback delays in multi-agent learning.
\newblock \emph{Advances in Neural Information Processing Systems}, 30,
  2017{\natexlab{a}}.

\bibitem[Zhou et~al.(2017{\natexlab{b}})Zhou, Mertikopoulos, Moustakas, Bambos,
  and Glynn]{zhou_mirror_2017}
Zhou, Z., Mertikopoulos, P., Moustakas, A.~L., Bambos, N., and Glynn, P.
\newblock Mirror descent learning in continuous games.
\newblock In \emph{2017 {IEEE} 56th {Annual} {Conference} on {Decision} and
  {Control} ({CDC})}, pp.\  5776--5783, December 2017{\natexlab{b}}.
\newblock \doi{10.1109/CDC.2017.8264532}.

\bibitem[Zhou et~al.(2018)Zhou, Mertikopoulos, Athey, Bambos, Glynn, and
  Ye]{zhou2018learning}
Zhou, Z., Mertikopoulos, P., Athey, S., Bambos, N., Glynn, P.~W., and Ye, Y.
\newblock Learning in games with lossy feedback.
\newblock \emph{Advances in Neural Information Processing Systems}, 31, 2018.

\bibitem[Zhou et~al.(2020)Zhou, Mertikopoulos, Bambos, Boyd, and
  Glynn]{zhou_convergence_2020}
Zhou, Z., Mertikopoulos, P., Bambos, N., Boyd, S.~P., and Glynn, P.~W.
\newblock On the {Convergence} of {Mirror} {Descent} beyond {Stochastic}
  {Convex} {Programming}.
\newblock \emph{SIAM Journal on Optimization}, 30\penalty0 (1):\penalty0
  687--716, January 2020.
\newblock ISSN 1052-6234.
\newblock \doi{10.1137/17M1134925}.
\newblock URL \url{https://epubs.siam.org/doi/abs/10.1137/17M1134925}.
\newblock Publisher: Society for Industrial and Applied Mathematics.

\bibitem[Zinkevich(2003)]{zinkevich2003online}
Zinkevich, M.
\newblock Online convex programming and generalized infinitesimal gradient
  ascent.
\newblock In \emph{Proceedings of the 20th international conference on machine
  learning (ICML)}, 2003.

\end{thebibliography}
\bibliographystyle{icml2023}

%%%%%%%%%%%%%%%%%%%%%%%%%%%%%%%%%%%%%%%%%%%%%%%%%%%%%%%%%%%%%%%%%%%%%%%%%%%%%%%
%%%%%%%%%%%%%%%%%%%%%%%%%%%%%%%%%%%%%%%%%%%%%%%%%%%%%%%%%%%%%%%%%%%%%%%%%%%%%%%
% APPENDIX
%%%%%%%%%%%%%%%%%%%%%%%%%%%%%%%%%%%%%%%%%%%%%%%%%%%%%%%%%%%%%%%%%%%%%%%%%%%%%%%
%%%%%%%%%%%%%%%%%%%%%%%%%%%%%%%%%%%%%%%%%%%%%%%%%%%%%%%%%%%%%%%%%%%%%%%%%%%%%%%
\newpage
\appendix
\onecolumn
\notshow{
\section{Potential Societal Impact}\label{sec:sociatal impact}
This work provides theoretical results for the convergence rate of  online learning algorithms in multi-player games. 
Online learning in multi-player games is a mathematical model that captures the strategic interaction between agents in multi-agent systems. From this perspective, our convergence results provide new  understandings of the evolution of the overall behavior of agents in multi-agent systems. More specifically, our results imply that certain natural dynamics will lead the agents' joint action profile to a stable state, i.e., a Nash equilibrium, efficiently. As a direct application, a designer of a multi-agent system can prescribe the learning algorithms studied in this paper, i.e., accelerated optimistic gradient, to agents, so that the system stabilizes quickly. Moreover, practical applications of min-max optimization (a special case of the games studied in this paper) include Generative Adversarial Networks (GANs) and adversarial examples. Therefore, our results might also provide useful insights on the training of GANs and adversarial training. To our best knowledge, we do not envision any immediate negative societal impacts of our results, such as security, privacy, and fairness issues.}

\section{Missing proofs in Section 3}
\label{app:single-step}
\begin{prevproof}{Lemma}{lem:single-step}
    Let us view the update rule of \ref{AOG} as standard update rule of \ref{OG} with modified gradients $g_{t-\half}- \frac{1}{\eta_t(t+1)}(x_1 - x_t)$ and $g_{t+\half} - \frac{1}{\eta_t(t+1)}(x_1 - x_t)$. Thus by the standard analysis of \ref{OG} (see~\cite{rakhlin2013optimization}[Lemma 1]), we have for any $t\ge 1$ and any $x' \in \X$, 
\begin{align*}
    \InAngles{g_{t+\half}-\frac{1}{\eta_t(t+1)}(x_1 - x_t), x_{t+\half} - x'} &\le \frac{1}{2\eta_t} \InParentheses{\InNorms{x_t- x'}^2 - \InNorms{x_{t+1} - x'}^2} + \InNorms{g_{t+\half} - g_{t-\half}} \cdot \InNorms{x_{t+\half} - x_{t+1}}\\
    & \le  \frac{1}{2\eta_t} \InParentheses{\InNorms{x_t- x'}^2 - \InNorms{x_{t+1} - x'}^2} + \eta_t\InNorms{g_{t+\half} - g_{t-\half}}^2, 
\end{align*}
where in the second inequality we use the following inequality:
\begin{align*}
    \InNorms{x_{t+\half} - x_{t+1}} &\le \InNorms{\Pi_\X\InBrackets{x_t-\eta_t g_{t-\half}- \frac{1}{t+1}(x_1 - x_t)} -  \Pi_\X\InBrackets{x_t-\eta_t g_{t+\half}- \frac{1}{t+1}(x_1 - x_t)} } \\
    & \le \eta_t \InNorms{g_{t+\half} - g_{t-\half} }. \tag{$\Pi_X$ is non-expansive}
\end{align*}
Therefore, we can bound the single-step regret by
\begin{align*}
    \InAngles{g_{t+\half}, x_{t+\half} - x'} &\le  \frac{1}{2\eta_t} \InParentheses{\InNorms{x_t- x'}^2 - \InNorms{x_{t+1} - x'}^2} + \eta_t\InNorms{g_{t+\half} - g_{t-\half}}^2 + \InAngles{\frac{1}{\eta_t(t+1)}(x_1 - x_t), x_{t+\half} - x'} \\ 
    &\le  \frac{1}{2\eta_t} \InParentheses{\InNorms{x_t- x'}^2 - \InNorms{x_{t+1} - x'}^2} + \eta_t\InNorms{g_{t+\half} - g_{t-\half}}^2 + \frac{D^2}{\eta_t(t+1)},
\end{align*}
where in the last inequality we use Cauchy-Schwarz inequality and the fact that $\X$ is bounded by $D$. This completes the proof.
\end{prevproof}

\begin{prevproof}{Theorem}{thm:regret}
     Let $T_1 \ge 2$ be the last time the player uses constant step size $\eta$. By line 7 of \Cref{alg}, we know the the second-order gradient variation $S_{T_1+1} \le S_{T_1} + 2G^2$ is upper bounded by a constant. By telescoping the inequality from \Cref{lem:single-step}, we know that the player's regret up to time $T_1$ is at most
    \begin{align*}
        &\sum_{t=1}^{T_1}\InAngles{g_{t+\half}, x_{t+\half} - x' } \\
        &\le \frac{\InNorms{x_1-x'}^2}{2\eta} + \eta S_{T_1+1} + G^2  + \sum_{t=1}^{T_1} \frac{D^2}{\eta (t+1)} \\
        & \le \mathcal{O}(G^2 + \log T_1). 
    \end{align*}
    Now we consider $t\ge T_1 + 1$ when the player switches to an adaptive step size. Using \Cref{lem:single-step}, for any $T \ge T_1 +1$, we have 
    \begin{align*}
        &\sum_{t=T_1+1}^{T}\InAngles{g_{t+\half}, x_{t+\half} - x'} \\
        &\le \underbrace{\sum_{t=T_1+1}^T  \frac{1}{\eta_t} (\InNorms{x_t - x^*}^2 - \InNorms{x_{t+1}- x^*}^2)}_{\mathbf{I}} \\
        & \quad\quad + \underbrace{\sum_{t=T_1+1}^T \eta_t \InNorms{g_{t+\half} - g_{t-\half}}^2}_{\mathbf{II}} + \underbrace{\sum_{t=T_1+1}^T \frac{D^2}{\eta_t (t+1)}}_{\mathbf{III}}.
    \end{align*}
    Since for any $t\ge 1$, $\InNorms{g_{t+\half}-g_{t-\half}}^2 \le2\InNorms{g_{t+\half}}^2+2\InNorms{g_{t-\half}}^2 \le 4G^2$. We have  $S_t \le 4G^2t$ and $\eta_t = \frac{1}{\sqrt{1+S_t}} \ge \frac{1}{2G\sqrt{t}}$ for any $t\ge T_1+1$. We now proceed to bound each terms as follows.
    \begin{align*}
        \mathbf{I} \le \frac{D^2}{\eta_{T_1+1}} + \sum_{t=T_1+2}^T D^2 \InParentheses{\frac{1}{\eta_t} - \frac{1}{\eta_{t-1}}} \le \frac{D^2}{\eta_T} \le \mathcal{O}(D^2G\sqrt{T}).
    \end{align*}
    \begin{align*}
        \mathbf{II} &= \sum_{t=T_1+1}^T \InParentheses{ \eta_{t+1} + \eta_t - \eta_{t+1} }\InNorms{g_{t+\half} - g_{t-\half}}^2 \\
        & \le \sum_{t=T_1+1}^T \InParentheses{ \frac{  \InNorms{g_{t+\half} - g_{t-\half}}^2 }{ \sqrt{1 + S_{t+1} }} + 4G^2(\eta_t - \eta_{t+1})}\\
        & \le \sum_{t=T_1+1}^T  \frac{  (\sqrt{1+S_{t+1}}-\sqrt{1+S_t})(\sqrt{1+S_{t+1}}+\sqrt{1+S_t}) }{ \sqrt{1 + S_{t+1} }} + 4G^2\\
        & \le \sum_{t=T_1+1}^T   2(\sqrt{1+S_{t+1}}-\sqrt{1+S_t}) + 4G^2\\
        & \le  2\sqrt{1+\sum_{t=1}^{T} \InNorms{g_{t+\half} -g_{t-\half} }^2} + 4G^2 = \mathcal{O}(G\sqrt{T} + G^2).
    \end{align*}
    %where we use \Cref{lem:sequnce} in the last inequality.
    %\yang{Let's don't reference \Cref{lem:sequnce}.}
    \begin{align*}
        \mathbf{III} & \le D^2 \sum_{i=1}^t \frac{\sqrt{1+S_t}}{t+1} \le D^2 \sum_{t=1}^T \mathcal{O}(\frac{G}{\sqrt{t}}) = \mathcal{O}(D^2G\sqrt{T}).
    \end{align*}
    Combing the above inequalities, we get the regret between $T_1$ and $T$ is at most $\mathcal{O}(D^2G\sqrt{T} + G^2)$.
\end{prevproof}

\section{Missing proofs in \Cref{sec:last-iterate rate}}
\label{app:last-iterate}

\begin{prevproof}{Proposition}{prop:P2}
    Note that $\boldx_{3/2} = \boldx_1$ and $\eta c_2 = \boldx_1 - \eta V(\boldx_1) - \boldx_2$. Thus 
    \begin{align*}
        \InNorms{\eta V(\boldx_2) + \eta c_2} &= \InNorms{\eta V(\boldx_2) + \boldx_1 - \eta V(\boldx_1) - \boldx_2}\\
        & \le \eta \InNorms{V(\boldx_2) - V(\boldx_1)} + \InNorms{\boldx_1 - \boldx_2} \\
        & \le (1+\eta L) \InNorms{\boldx_1 - \boldx_2}\tag{$V$ is $L$-Lipschitz} \\
        &\le \frac{3D}{2}. \tag{$\eta L \le \frac{1}{2}$} 
    \end{align*}
    Using the above inequality, we can bound $P_2$ as follows:
    \begin{align*}
         P_2 &=  3\InParentheses{ \InNorms{\eta V(\boldx_2)+\eta c_2}^2 + \InNorms{\eta V(\boldx_2) - \eta V(\boldx_{1})}^2} + 2 \InAngles{\eta V(\boldx_2) + \eta c_2, \boldx_2 - \boldx_1} \\
         &\le 3\InParentheses{ \InNorms{\eta V(\boldx_2)+\eta c_2}^2 + \eta L \InNorms{\boldx_2 - \boldx_1}^2} + 2 \InNorms{\eta V(\boldx_2) + \eta c_2}\InNorms{\boldx_2 - \boldx_1}  \\
         &\le 3\InParentheses{ \frac{9D^2}{4} + \frac{D^2}{4}} + 3D^2\tag{$\eta L \le \frac{1}{2}$} \\
         & = \frac{33D^2}{4} \le 9D^2. 
    \end{align*}
    This completes the proof of \Cref{prop:P2}.
\end{prevproof}

\begin{prevproof}{Lemma}{lem:residual}
    Fix any $t \ge 2$. Using triangle inequality, we have
    \begin{align*}
        \InNorms{\boldx_{t+\half} - \boldx_{t}} \le \InNorms{\boldx_{t+\half} - \Pi_\X\InBrackets{ \boldx_t - \eta V(\boldx_t)}} +  \InNorms{\Pi_\X\InBrackets{\boldx_t - \eta V(\boldx_t)}-\boldx_{t}}.
    \end{align*}
    We can bound the first term as follows:
    \begin{align*}
        \InNorms{\boldx_{t+\half} - \Pi_\X\InBrackets{ \boldx_t - \eta V(\boldx_t)}} &= \InNorms{\Pi_\X\InBrackets{\boldx_t - \eta V(\boldx_{t-\half}) + \frac{1}{t+1}(\boldx_1-\boldx_t)} - \Pi_\X\InBrackets{ \boldx_t - \eta V(\boldx_t)}} \\
        & \le \InNorms{\eta V(\boldx_t) - \eta V(\boldx_{t-\half}) + \frac{1}{t+1}(\boldx_1 -\boldx_t)} \tag{$\Pi_X$ is non-expansive} \\
        & \le \InNorms{\eta V(\boldx_t) - \eta V(\boldx_{t-\half})}  + \frac{\InNorms{\boldx_1 - \boldx_t}}{t+1} \\
        & \le \frac{14D}{t}.  \tag{\Cref{lem:ls-tangent}}
    \end{align*}
    Since $c_t \in N_\X(\boldx_t)$, we have $\boldx_t = \Pi_\X\InBrackets{\boldx_t + \eta c_t}$. Using this fact we can bound the second term: 
    \begin{align*}
        \InNorms{\Pi_\X\InBrackets{\boldx_t - \eta V(\boldx_t)}-\boldx_{t}} &= \InNorms{\Pi_\X\InBrackets{\boldx_t - \eta V(\boldx_t)}-\Pi_\X\InBrackets{\boldx_t + \eta c_t}} \\
        &\le \InNorms{\eta V(\boldx_t) + \eta c_t} \tag{$\Pi_X$ is non-expansive} \\
        & \le \frac{13D}{t}. \tag{\Cref{lem:ls-tangent}} 
    \end{align*}
    Combing the above inequalities, we have $\InNorms{\boldx_{t+\half} - \boldx_{t}} \le \frac{27D}{t}$. This completes the proof of \Cref{lem:residual}.
\end{prevproof}

\section{Last-Iterate Convergence Rate without the Boundedness Assumption}\label{app:unbounded domain}
Recall that we prove $r^{tan}(\boldx_{T+\half}) \le \frac{55 D}{\eta T}$ for all $T \ge 2$ in \Cref{thm:last-iterate rate} with the assumption that the action set $\X$ is bounded by $D > 0$. In this section, we prove  last-iterate convergence rate of \ref{AOG}, which is similar to \Cref{thm:last-iterate rate} but without the boundedness assumption on $\X$. 

\begin{theorem}
\label{thm:last-iterate rate-unbounded domain}
Let $\mathcal{G}=\{N,(\X^i)_{i \in [N]},(\ell^i)_{i\in[N]}\}$ be a $L$-smooth monotone game, where each player $i$'s action set $\X^i \subseteq \R^{n_i}$ is convex and closed, but not necessarily compact. When all players employ \ref{AOG} with a constant step size $\eta \le \frac{1}{\sqrt{6}L}$ in $\mathcal{G}$, then for any $T \ge 2$, we have
\[
    r^{tan}(\boldx_{T+\half}) \le \frac{1430H}{\eta T},
\]
where $H = \max\{ \eta \cdot r^{tan}(\boldx_1),\InNorms{\boldx_1 - \boldx_\star}\}$ with $\boldx_\star$ being an Nash equilibrium of $\mathcal{G}$.
\end{theorem}
\notshow{
Specifically, denote $H = \max\{ \eta \cdot r^{tan}(\boldx_1),\InNorms{\boldx_1 - \boldx_\star}\}$, we are going to prove that for any $T \ge 2$, 
\[ r^{tan}(\boldx_{T+\half}) \le \frac{55 \times 26 H}{ \eta T} = \frac{1430 H}{\eta T}.\]
\yangnote{Yang: could you add the modified statement of \Cref{thm:last-iterate rate} here in the theorem environment?}}
\begin{proof} We will go through the proof of the first part of \Cref{thm:last-iterate rate} (including \Cref{prop:P2}, \Cref{lem:ls-tangent}, and \Cref{lem:residual}), check every application of the boundedness assumption, and give upper bound on these terms using $H$. 

In the proof of \Cref{prop:P2}, the boundedness assumption is used to bound $\InNorms{\boldx_1 - \boldx_2}$. Here we show that it is upper bounded by $H$. Let $c \in N_\X(\boldx_1)$ be any vector in the normal cone $N_\X(\boldx)$. Then we have \[ \InNorms{\boldx_1 - \boldx_2} \le \InNorms{ \Pi_\X [\boldx_1 - \eta c] } - \Pi_\X[\boldx_1 - \eta V(\boldx_1)] \le \eta \InNorms{V(\boldx_1) + c}.\] Thus $\InNorms{\boldx_1 - \boldx_2} \le \eta \cdot r^{tan}(\boldx_1) \le H$. 

In the proof of \Cref{lem:ls-tangent}, the boundedness assumption is applied to bound $\InNorms{\boldx_1 - \boldx_\star}$, which is clearly upper bounded by $H$. Combing the above two observations, we get a modified version of \Cref{lem:ls-tangent} (replacing $D$ with $H$): for all $t \ge 2$, 
\[
    \InNorms{V(\boldx_t) + c_t} \le \frac{13H}{\eta t}, \quad \InNorms{V(\boldx_t) - V(\boldx_{t- \half})} \le \frac{13H}{\eta t}.
\]
Using triangle inequality, this further implies that for all $t \ge 1$,
\[ 
\InNorms{V(\boldx_{t+\half}) + c_{t+1}} \le \InNorms{V(\boldx_{t+1}) + c_{t+1}} + \InNorms{V(\boldx_{t+1}) - V(\boldx_{t+\half})} \le \frac{26 H}{\eta (t+1)}.
\]

In the proof of \Cref{lem:residual} and the remaining proof of \Cref{thm:last-iterate rate}, the boundedness assumption is applied to $\{ \InNorms{\boldx_1 - \boldx_t}\}_{t \in [T]}$. We show how to bound $\InNorms{\boldx_1 - \boldx_t}$ for $t \ge 3$ as follows. 

\paragraph{Bounding $\InNorms{\boldx_1 - \boldx_t}$}
Using the update rule and the definition of $c_{t+1}$, we have the following identity: \[\InNorms{\boldx_{t+1} - \boldx_1}^2 = \InNorms{\frac{t}{t+1}(\boldx_t - \boldx_1) - \eta (V(\boldx_{t+1/2}) + c_{t+1})}^2, \quad \forall t \ge 2.\] Thus $\InNorms{\boldx_{t+1} - \boldx_1}^2$ is upper bounded by $\frac{t^2}{(t+1)^2}(1+\frac{1}{t})\InNorms{\boldx_t - \boldx_1}^2 + (1+t) \InNorms{\eta (V(\boldx_{t+1/2}) + c_{t+1})}^2$ using Young's inequality. Recall that we just get $\InNorms{V(\boldx_{t+1/2}) + c_{t+1}}\le \frac{26 H}{\eta (t+1)}$ using the modified Lemma 4. Combing the above inequalities gives $\InNorms{\boldx_{t+1} - \boldx_1}^2 \le \frac{t}{t+1} \InNorms{\boldx_t - \boldx_1}^2 + \frac{26^2 H^2}{t+1}$, which is equivalent to $(t+1)\InNorms{\boldx_{t+1} - \boldx_1}^2 \le t\InNorms{\boldx_t - \boldx_1}^2 + 26^2 H^2$. Telescoping the above inequality gives $\InNorms{\boldx_t - \boldx_1}^2 \le  \frac{2 \InNorms{\boldx_2 - \boldx_1}^2 + 26^2 H^2 (t-2)}{t} \le 26^2 H^2$ for all $t \ge 3$. Thus $\InNorms{\boldx_t - \boldx_1} \le 26 H$ for all $t \ge 3$. 

Now we have upper bounded every terms where the boundedness assumption is applied in the proof of the first part of \Cref{thm:last-iterate rate} by $26 H$. Replacing $D$ with $26H$ in the first part of \Cref{thm:last-iterate rate} completes the proof.
\end{proof}

\section{Proof of \Cref{thm:second-order gradient variation}}\label{app:second-order}
\begin{prevproof}{Theorem}{thm:second-order gradient variation}
    In the game setting, player $i$'s second-order gradient variation is $S_T^i=\sum_{t=2}^T  \InNorms{V^i(\boldx_{t+\half})-V^i(\boldx_{t-\half})}^2$. Using \Cref{lem:ls-tangent} and \Cref{lem:residual}, we have
    \begin{align*}
        \InNorms{V^i(\boldx_{t+\half})-V^i(\boldx_{t-\half})}^2 &\le \InNorms{V(\boldx_{t+\half})-V(\boldx_{t-\half})}^2 \\
        &\le 2L^2 \InNorms{\boldx_{t+\half} - \boldx_t}^2 + 2\InNorms{V(\boldx_t) -V(\boldx_{t-\half})}^2 \tag{$L$-Lipschitzness of $V$} \\
        &\le \frac{2L^2 \cdot 27^2 D^2}{t^2} + \frac{2\cdot 13^2 D^2}{\eta^2 t^2} \\
        & = \frac{(1458L^2 + \frac{338}{\eta^2}) D^2}{t^2}.
    \end{align*}
    For a choice of $\eta = \frac{1}{3L}$, we have %$\InNorms{g^i_{t+\half}-g^i_{t-\half}}^2 \le \frac{4500D^2L^2}{t^2}$
    $$\InNorms{V^i(\boldx_{t+\half})-V^i(\boldx_{t-\half})}^2 \le \frac{4500D^2L^2}{t^2}$$ and $S^i_T \le 4500\pi D^2 L^2$.
\end{prevproof}

\section{Linear Regret of EAG}
\label{app:EAG}
In this section, we review the definition of the Extra Anchored Gradient \eqref{EAG} algorithm and show that it is not a no-regret algorithm when implemented it in the online learning setting. The proof is similar to the linear regret proof of EG~\cite{golowich2020tight} and we include it for completeness. Given a game $\mathcal{G}$ with gradient operator $V$, initial point $x_1 \in \X $, the Extra Anchored Gradient algorithm updates as follows:
\begin{equation}
\label{EAG}
\tag{EAG}
\begin{aligned}
    x_{t+\half} &= \Pi_\X \InBrackets{x_t - \eta V(x_t) + \frac{1}{t+1}(x_1 - x_t)}, \\
    x_{t+1}     &= \Pi_\X \InBrackets{x_t - \eta V(x_{t+\half}) +\frac{1}{t+1}(x_1 - x_t)}.
\end{aligned}
\end{equation}
The key difference of \ref{EAG} compared to \ref{AOG} is that in one iteration, the update of \ref{EAG} requires two gradients $V(x_t)$ and $V(x_{t+\half})$. Since in online learning setting, players only see the gradients corresponding to the action they play, players must play both $x_t$ and $x_{t+\half}$ using \ref{EAG}. Thus to implement \ref{EAG} in standard online learning setting, we need two iterations for each iteration of \ref{EAG}. Specifically, each player $i$ plays $y^i_t$ for $t \ge 1$, while $y^i_{2t-1} = x^i_t$ and $y^i_{2t} = x^i_{t+\half}$. The corresponding update is for $t \ge 1$,
\begin{align}
    y^i_{2t} &= \Pi_{\X^i} \InBrackets{y^i_{2t-1} - \eta V^i(y_{2t-1}) + \frac{1}{t+1}(y^i_1 - y^i_{2t-1})}, \label{EAG-1} \\
     y^i_{2t+1} &= \Pi_{\X^i} \InBrackets{y^i_{2t-1} - \eta V^i(y_{2t}) + \frac{1}{t+1}(y^i_1 - y^i_{2t-1})}. \label{EAG-2}
\end{align}
We will show when the other players' action $y^{-i}_t$ is adversarial, \ref{EAG} has linear regret and is not no-regret. 
\begin{proposition}
    There exits a two-player zero-sum $1$-smooth game $\mathcal{G} = ([2], \{\X_1, \X_2\}, (f,-f))$, such that for an adversarial choice of $(y^2_t)_{t \in [T]}$, the \ref{EAG} updates \eqref{EAG-1} and \eqref{EAG-2} for the first player has $\Omega(T)$ regret for any $T \ge 1$.  
\end{proposition}
\begin{proof}
We use exactly the same construction as~\cite{golowich2020tight}[Proposition 10]. We take $\X^1 = \X^2 = [-1,1]$ and $f: \X \rightarrow \R$ to be $f(y^1,y^2) = y^1 \cdot y^2$. Player 2 play the following sequence of actions:
\begin{align*}
   y^2_t = \begin{cases}
        1  & \text{$t$ is odd}   \\
        0  & \text{$t$ is even}
    \end{cases} 
\end{align*}
Then for any $t \ge 1$, we have 
\begin{align*}
    V^1(y_{2t-1}) =  y^2_{2t-1} = 1,\\
    V^1(y_{2t}) =  y^2_{2t} = 0.
\end{align*}
Suppose $y^1_1 = 0$. Then we have $y^1_{2t-1} = 0$ and $y^1_{2t} = \max\{-\eta, -1\}$ for any $t \ge 1$.  Thus the accumulative loss for player 1 until $T \ge 1$ round is $\sum_{t=1}^T f(y^1_t, y^2_t) = 0$. However, the accumulative loss of action $y^1 = -1$ is only $\sum_{t=1}^T f(-1, y^2_t) \le -\frac{T}{2}$. Thus the regret is at least $\frac{T}{2} = \Omega(T)$
\end{proof}

\section{Details on Numerical Experiments}
\label{app:experiment}
We choose 
\begin{align*}
    A = \frac{1}{4}\begin{bmatrix} 
    & & &-1 &1 \\
    & & \cdots & \cdots& \\
    & -1& 1&   &  \\
    -1& 1& &   & \\
    1& & &   & 
    \end{bmatrix} \in \R^{n \times n}, \quad b = \frac{1}{4} \begin{bmatrix}
        1\\
        1\\
        \cdots \\
        1\\
        1
    \end{bmatrix} \in \R^{n}, \quad h = \frac{1}{4} \begin{bmatrix}
        0\\
        0\\
        \cdots \\
        0\\
        1
    \end{bmatrix} \in \R^{n},
\end{align*}
and $H = 2A^\top A$. As shown in \citep{ouyang2021lower}, $\InNorms{A} \le \frac{1}{2}$ and $\InNorms{H} \le \frac{1}{2}$ which implies $f = \frac{1}{2}x^\top Hx - h^\top x - \InAngles{Ax -b, y}$ is $1$-smooth.  We choose $n = 100$, $\X=\Y =[-200,200]^n$.  We run both \ref{AOG} and \ref{OG} with step size $\eta = 0.3$ and initial points $x_1 = y_1 = \frac{1}{n} \boldsymbol{1}$ for $10^5$ iterations. The code can be found at \url{https://github.com/weiqiangzheng1999/Doubly-Optimal-No-Regret-Learning}.

\section{Auxiliary Results}
\label{app:auxiliary}

\notshow{
The following Lemma follows from \cite{auer2002adaptive}[Lemma 3.5].
\begin{lemma}\label{lem:sequnce}
    For any non-negative real numbers $v_1, v_2, \cdots, v_T$, it holds
    \begin{align*}
        \sum_{t=1}^T \frac{v_t}{\sqrt{1+\sum_{i=1}^t v_i}} \le 2\sqrt{\sum_{t=1}^T v_t}.
    \end{align*}
\end{lemma}}

\begin{proposition}
\label{prop:identity}
    In the setup of \Cref{lem:potential}, the following identity holds.
    \begin{align*}
        &P_t - P_{t+1} - t(t+1) \cdot \text{LHSI}~\eqref{eq:potential-1} - \frac{t(t+1)}{4q} \cdot \text{LHSI}~\eqref{eq:potential-2} \nonumber \\
        &\quad - t(t+1) \cdot \text{LHSI}~\eqref{eq:potential-3} - \frac{t(t+1)}{2} \cdot \InParentheses{\text{LHSI}~\eqref{eq:potential-4} + \text{LHSI}~\eqref{eq:potential-5} + \text{LHSI}~\eqref{eq:potential-6}} \nonumber \\
        & = \frac{t(t+1)}{2} \InNorms{\frac{\boldx_{t+\half} -x_t}{2} + \eta V(\boldx_t) - \eta V(\boldx_{t+\half})}^2 \nonumber \\
        & \quad + \frac{(1-4q)t-4q}{4q} (t+1) \InNorms{\eta V(\boldx_{t+\half}) - \eta V(\boldx_{t+1})}^2  \\
        & \quad + (t+1) \cdot \InAngles{\eta V(\boldx_{t+\half}) - \eta V(\boldx_{t+1}), \eta V(\boldx_{t+1}) + \eta c_{t+1}}.
    \end{align*}
\end{proposition}
\begin{proof}
We use {\sc MATLAB} to verify the following inequality, which implies the claim by suitable change of variables. For any vectors $a_0, a_1, a_2, a_3, a_4, b_1, b_2, b_3, b_4, u_2, u_4 \in \R^n$, any real numbers $t \ge 1$ and $q > 0$, if 
\begin{align*}
    a_4 = a_2 - b_3 + \frac{1}{t+1}(a_0 - a_2) - u_4,
\end{align*}
then the following identity holds 
\begin{align*}
    &\frac{t(t+1)}{2}\InParentheses{ \InNorms{a_2 + u_2}^2 + \InNorms{b_2 -b_1}^2} + t \InAngles{b_2 + u_2, a_2 - a_0} \\
    &- \frac{(t+1)(t+2)}{2}\InParentheses{ \InNorms{a_4 + u_4}^2 + \InNorms{b_4 -b_3}^2} + t \InAngles{b_4 + u_4, a_4 - a_0} \\
    &- t(t+1) \InAngles{b_4 - b_2, a_4 - a_2} - \frac{t(t+1)}{4q} \InParentheses{ q \InNorms{a_4 - a_3}^2 - \InNorms{b_4 - b_3}^2} \\
    & - t(t+1) \InAngles{u_4, a_4 - a_2} - \frac{t(t+1)}{2} \InAngles{u_2, a_2 - a_3} - \frac{t(t+1)}{2} \InAngles{u_2, a_2 - a_4} \\
    & - \frac{t(t+1)}{2} \InAngles{a_2 - b_1 + \frac{1}{t+1}(a_0 - a_2) -a_3, a_3 - a_4}\\
    =& \frac{t(t+1)}{2} \InNorms{\frac{a_3-a_4}{2} + b_1 - b_2}^2 \\
    & + \frac{t(t+1)}{2} \InNorms{\frac{a_3+a_4}{2}-a_2+ b_2 + u_2 - \frac{a_0 - a_2}{t+1}}^2 \\
    & + \frac{(1-4q)t-4q}{4q} (t+1) \InNorms{b_3 - b_4}^2 \\
    & + (t+1) \InAngles{b_3 - b_4, b_4 + u_4}.
\end{align*}
The {\sc MATLAB} code for verification of the above identity is available at \url{https://github.com/weiqiangzheng1999/Doubly-Optimal-No-Regret-Learning}. To see how the above identity implies the claimed identity, we replace $a_0$ with $x_1$; replace $a_k$ with $x_{t-1+\frac{k}{2}}$ for $k \in [4]$; replace $b_k$ with $\eta V(x_{t-1+\frac{k}{2}})$ for $k \in [4]$; replace $u_2$ with $\eta c_t$; replace $u_4$ with $\eta c_{t+1}$; and note that by the definition of $c_{t+1}$, we have 
\begin{align*}
    x_{t+1} = x_t - \eta V(x_{t+\half}) + \frac{1}{t+1}(x_1 -x_t) - \eta c_{t+1}. 
\end{align*}
This completes the proof.
\end{proof}

\begin{proposition}[\cite{cai2022accelerated}]
\label{prop:sequence analysis}
Let $\{a_k \in \R^{+}\}_{k\ge 2}$ be a sequence of real numbers. Let $C_1 \ge 0$ and $p \in (0, \frac{1}{3})$ be two real numbers. If for every $k\geq 2$,
    $\frac{k^2}{4} \cdot a_k \le C_1 + \frac{p}{1-p} \cdot \sum_{t=2}^{k-1} a_t$,
then for each $k\geq 2$ we have 
\begin{align*}%\label{eq:induction assumption}
    a_k \le \frac{4\cdot C_1}{1-3p} \cdot \frac{1}{k^2}.
\end{align*}
\end{proposition}

\end{document}